\documentclass[conference]{IEEEtran}
\usepackage{times}

\usepackage[numbers]{natbib}
\usepackage{graphicx}
\usepackage{epsfig} 
\usepackage{mathptmx}
\usepackage{amsmath} 
\usepackage{leftindex} 
\usepackage{bm}

\usepackage{amssymb}  
\usepackage{amsthm}
\usepackage{thmtools}
\usepackage{thm-restate}
\usepackage{wasysym}
\usepackage{mathrsfs}  
\usepackage{gensymb}
\usepackage{bbm}
\usepackage{color}
\usepackage{lipsum}
\usepackage[ruled,vlined,linesnumbered]{algorithm2e}
\usepackage[colorlinks=true,linkcolor=black,anchorcolor=black,citecolor=black,filecolor=black,menucolor=black,runcolor=black,urlcolor=black]{hyperref}
\usepackage{etoolbox}
\usepackage[table,xcdraw,dvipsnames]{xcolor}
\usepackage{multirow}
\usepackage{bbm}
\usepackage{listings}
\usepackage{svg}
\usepackage{subcaption}

\usepackage{siunitx} 

\usepackage[capitalise]{cleveref}
\crefname{equation}{Eq.}{Eqs.}
\crefname{figure}{Fig.}{Figs.}
\crefname{tabular}{Tab.}{Tabs.}
\crefname{section}{Sec.}{Secs.}
\crefname{algocf}{Alg.}{Algs.}
\crefname{algocfline}{Line}{Lines}
\crefname{lem}{Lemma}{Lemmata}
\usepackage{color, colortbl}

\SetKwInput{KwInput}{Input}
\SetKwInput{KwReturn}{Return}
\SetKwComment{Comment}{/* }{ */}

\usepackage{outlines}
\let\labelindent\relax
\usepackage{enumitem}
\setenumerate[2]{label=\alph*.}
\setenumerate[3]{label=\roman*.}

\usepackage{setspace}
\usepackage{xspace}

\makeatletter
\patchcmd{\@makecaption}
  {\scshape}
  {}
  {}
  {}
\makeatother

\newtheorem{defn}{Definition}
\newtheorem{rem}[defn]{Remark}
\newtheorem{lem}[defn]{Lemma}

\newtheorem{problem}[defn]{Problem}

\newcommand{\regtext}[1]{\mathrm{\textnormal{#1}}}
\newcommand{\ts}[1]{\textsuperscript{#1}}

\newcommand{\mc}[1]{\mathcal{#1}}
\newcommand{\lbl}[1]{_{\regtext{#1}}}
\newcommand{\idx}[1]{_{#1}}
\newcommand{\arridx}[1]{[#1]}

\newcommand{\local}{\lbl{local}}
\newcommand{\actual}[1]{(\jointangle\idx{#1})}

\newcommand{\st}{\regtext{ s.t. }}

\newcommand{\R}{\mathbb{R}}
\newcommand{\N}{\mathbb{N}}
\newcommand{\SO}{\regtext{SO}}

\newcommand{\vc}[1]{\mathbf{#1}} 
\newcommand{\mat}[1]{{\begin{bmatrix} #1 \end{bmatrix}}}

\newcommand{\zeros}{\vc{0}}

\newcommand{\eye}{\vc{I}}
\newcommand{\rotmat}{\vc{R}}
\newcommand{\transvec}{\vc{p}}

\newcommand{\pinv}{^{\dagger}}
\newcommand{\norm}[1]{\left\Vert#1\right\Vert}
\newcommand{\abs}[1]{\left\vert#1\right\vert}
\newcommand{\trans}{^{\top}}
\newcommand{\transbrac}[1]{\left[ {#1} \right]\trans}

\DeclareMathOperator*{\argmin}{arg\,min}
\DeclareMathOperator{\skewmat}{sk}
\DeclareMathOperator{\atantwo}{atan2}

\newcommand{\rodrigues}{\mc{R}}
\newcommand{\subproblem}[1]{\method{SP#1}}
\DeclareMathOperator{\normalize}{unit}

\newcommand{\metric}{\mu}
\newcommand{\similarity}{\metric\lbl{c}}
\newcommand{\matrixsimilarity}{\metric\lbl{m}}

\newcommand{\action}{\vc{a}}
\newcommand{\observation}{\vc{o}}
\DeclarePairedDelimiter{\ceil}{\lceil}{\rceil}

\newcommand{\method}[1]{\textnormal{\texttt{#1}}\xspace}
\newcommand{\ourmethod}{\method{SEW-Mimic}}
\newcommand{\basemink}{\method{MINK-IK}}
\newcommand{\boundjointangles}{\method{BoundJoints}}
\newcommand{\alignwrist}{\method{AlignWrist}}
\newcommand{\alignaxis}{\method{AlignAxis}}
\newcommand{\eulerangle}{\method{EulerDecomp}}
\newcommand{\makeframe}{\method{MakeFrame}}

\newcommand{\config}{\vc{q}}
\newcommand{\axis}{\vc{h}}

\newcommand{\jointangle}{q}

\newcommand{\robot}{\regtext{rb}}
\newcommand{\human}{\regtext{hm}}

\newcommand{\streaminput}{\regtext{in}}

\newcommand{\opt}{^{\star}}
\newcommand{\leftside}{\regtext{lf}}
\newcommand{\rightside}{\regtext{rt}}
\newcommand{\lside}[1][]{_{\leftside\if\relax\detokenize{#1}\relax\else,\,#1\fi}}
\newcommand{\rside}[1][]{_{\rightside\if\relax\detokenize{#1}\relax\else,\,#1\fi}}

\newcommand{\init}{\lbl{0}}
\newcommand{\des}{\lbl{des}}
\newcommand{\safe}{\lbl{safe}}
\newcommand{\interp}{\lbl{interp}}

\newcommand{\shoulder}{\vc{s}}
\newcommand{\elbow}{\vc{e}}
\newcommand{\wrist}{\vc{w}}
\newcommand{\tool}{\vc{t}}
\newcommand{\indexfinger}{\vc{p}\lbl{index}}
\newcommand{\pinkyfinger}{\vc{p}\lbl{pinky}}
\newcommand{\handorientation}{\vc{H}}
\newcommand{\toolorientation}{\vc{T}}
\newcommand{\upperarm}{\vc{u}}
\newcommand{\lowerarm}{\vc{l}}
\newcommand{\orthoframe}[1]{\rotmat\frms{0}{#1}}
\newcommand{\frameorigin}[1]{\transvec\frms{0}{#1}}

\newcommand{\safetyfilter}{\method{SEW-SafetyFilter}}
\newcommand{\FK}{\method{FK}}
\newcommand{\linspace}{\method{linspace}}
\newcommand{\capsulefunc}{\method{Capsule}}
\newcommand{\makecapsules}{\method{MakeCapsules}}
\newcommand{\collisioncheck}{\method{CollisionCheck}}
\newcommand{\XPBDupdate}{\method{XPBD-Iter}}
\newcommand{\XPBDLengthUpdate}{\method{XPBD-Iter-L}}
\newcommand{\upperarmset}{U}
\newcommand{\lowerarmset}{L}
\newcommand{\handset}{H}
\newcommand{\torsoset}{T}
\newcommand{\radius}{r}
\newcommand{\collisionvolumes}{\mathcal{C}}
\newcommand{\capsule}{C}

\newcommand{\weight}{w}
\newcommand{\length}{\ell}
\newcommand{\keypoints}{K}
\newcommand{\robotkeypoints}[1][]{\keypoints_{\regtext{rb}\if\relax\detokenize{#1}\relax\else,\,#1\fi}}

\newcommand{\policy}{\pi}

\newcommand{\recovertoolori}{\method{RecoverTool}}
\newcommand{\conttimecollision}{\method{FindFirstCollision}}

\newcommand{\frm}[1]{^{#1}}
\newcommand{\frms}[2]{^{#1,#2}}

\title{\LARGE \bf
A Closed-Form Geometric Retargeting Solver for \\
Upper Body Humanoid Robot Teleoperation
}

\author{
    \authorblockN{
        Chuizheng Kong\authorrefmark{2},
        Yunho Cho\authorrefmark{2},
        Wonsuhk Jung\authorrefmark{2}, 
        Idris Wibowo\authorrefmark{1}\authorrefmark{2},
        Parth Shinde\authorrefmark{1}\authorrefmark{6},
        Sundhar Vinodh-Sangeetha\authorrefmark{1}\authorrefmark{3},\\
        Long Kiu Chung\authorrefmark{4},
        Zhenyang Chen\authorrefmark{2},
        Andrew Mattei\authorrefmark{4},
        Advaith Nidumukkala\authorrefmark{4} \\
        Alexander Elias\authorrefmark{7},
        Danfei Xu\authorrefmark{2}\authorrefmark{5},
        Taylor Higgins\authorrefmark{8}, and
        Shreyas Kousik\authorrefmark{2}\authorrefmark{4}
    }
    \authorblockA{\authorrefmark{1} denotes equal contribution}
    \authorblockA{\authorrefmark{2} Institute for Robotics and Intelligent Machines}
    \authorblockA{\authorrefmark{3} School of Aerospace Engineering}
    \authorblockA{\authorrefmark{4} Woodruff School of Mechanical Engineering}
    \authorblockA{\authorrefmark{5} School of Interactive Computing\\ 
    Georgia Institute of Technology, Atlanta, Georgia 30332}
    \authorblockA{\authorrefmark{6} Qualcomm, San Diego, California 92121}
    \authorblockA{\authorrefmark{7} Standard Bots, Glen Cove, New York 11542}
    \authorblockA{\authorrefmark{8} Department of Mechanical \& Aerospace Engineering \\
    Florida A\&M University-Florida State University,
    Tallahassee, Florida 32130}
}

\begin{document}

\makeatletter
\let\@oldmaketitle\@maketitle
\renewcommand{\@maketitle}{\@oldmaketitle
    \bigskip
    \includegraphics[width=\linewidth]
    {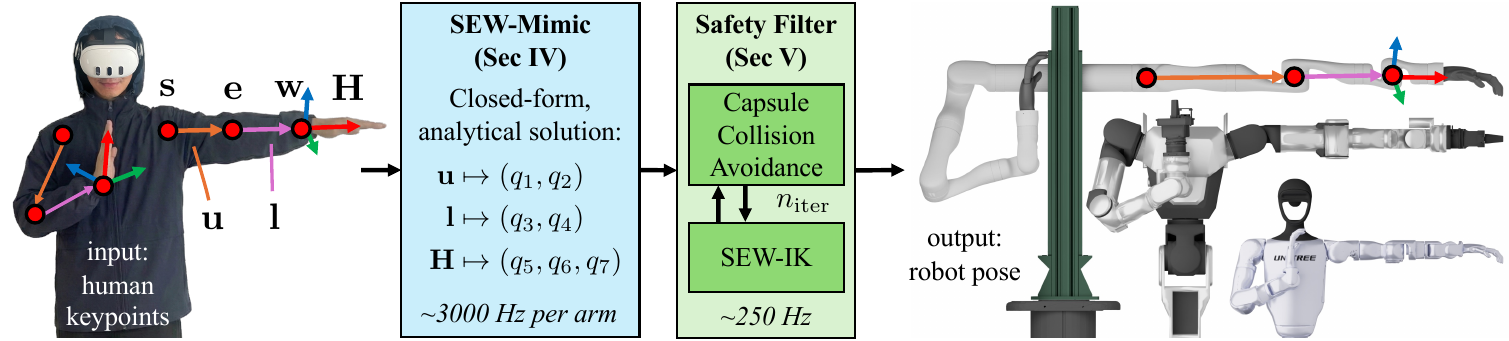}
    \captionof{figure}{We propose \ourmethod for retargeting human shoulder, elbow, and wrist (SEW) keypoints analytically to robot manipulator joint angles.
    Our key insight is to focus on limb and hand/tool \textit{orientation error} as a metric of pose similarity, because it enables a closed-form, provably-optimal retargeting solution that is scale independent, as shown here on Kinova Gen3, Rainbow RB-Y1, and Unitree G1 platforms (shown to scale).
    Since \ourmethod is computationally fast, we use it as an inverse kinematics (IK) solver within a safety filter to prevent self-collisions for bimanual teleoperation.}
    \label{fig: front figure}
    \vspace*{-0.75em}
    }
\makeatother

\maketitle
\thispagestyle{empty}
\pagestyle{empty}

\begin{abstract}
Retargeting human motion to robot poses is a practical approach for teleoperating bimanual humanoid robot arms, but existing methods can be suboptimal and slow, often causing undesirable motion or latency.
This is due to optimizing to match robot end-effector to human hand position and orientation, which can also limit the robot's workspace to that of the human.
Instead, this paper reframes retargeting as an orientation alignment problem, enabling a closed-form, geometric solution algorithm with an optimality guarantee.
The key idea is to align a robot arm to a human's upper and lower arm orientations, as identified from shoulder, elbow, and wrist (SEW) keypoints; hence, the method is called SEW-Mimic. 
The method has fast inference  (3 kHz) on standard commercial CPUs, leaving computational overhead for downstream applications; an example in this paper is a safety filter to avoid bimanual self-collision.
The method suits most 7-degree-of-freedom robot arms and humanoids, and is agnostic to input keypoint source.
Experiments show that SEW-Mimic outperforms other retargeting methods in computation time and accuracy.
A pilot user study suggests that the method improves teleoperation task success.
Preliminary analysis indicates that data collected with SEW-Mimic improves policy learning due to being smoother.
SEW-Mimic is also shown to be a drop-in way to accelerate full-body humanoid retargeting.
Finally, hardware demonstrations illustrate SEW-Mimic's practicality.
The results emphasize the utility of SEW-Mimic as a fundamental building block for bimanual robot manipulation and humanoid robot teleoperation.
\end{abstract}

\addtocounter{figure}{-1}

\section{Introduction}\label{sec: intro}

Retargeting is the process of mapping human poses to robot configurations. 
For robot teleoperation, this offers a human-centered way to control high degree-of-freedom (DoF) robots, by which training data can be collected for autonomous policies \cite{trilbmteam2025carefulexaminationlargebehavior}.
However due to physical differences between human and robot embodiments, such as size, DoFs, and motion constraints, the retargeting problem can be nontrivial.

As a result, most existing systems only track human hand motion \cite{jiang_dexmimicgen_2025, iyer_open_2024, ding_bunny-visionpro_2024, cheng_open-television_2024, unitreerobotics_xr_teleoperate}.
These methods solve retargeting as an end-effector inverse kinematics (IK) problem using robot's Jacobian Matrix \cite{khatib2003unified} whose inverse maps 6-DoF end-effector velocity to robot joint velocities (6-DoF / 7-DoF).
This approach prevents robot teleoperation near the limit of the robot's range of motion or near other joint singularity poses due to the inverse becoming degenerate. 
While degeneracy can be mitigated for 6-DoF arms using pseudo-inverse \cite{carpentier2019pinocchio} or optimization-based IK \cite{Zakka_Mink_Python_inverse_2025}, humanoid arms are typically 7-DoF to match human arms.
When using 6-DoF hand-only retargeting method, part of the 7-DoF arm corresponding to the human elbow joint can produce unwanted null-space motion \cite{elias_redundancy_2024}, increasing collision risk (see \Cref{fig: glass gap sew vs mink} in \Cref{sec: experiments}). 

To overcome limitations from hand-only retargeting, recent learning-based humanoid retargeting approaches optimize over robot joint angles to match the robot to human body keypoints (shoulder, elbow, wrist, ankles, etc.) \cite{ze_twist_2025, li2025clone, he_learning_2024}.
While minimizing the Euclidean distance between the human and robot keypoints includes the elbow, solving such a program can average 0.7 seconds \cite{ze_twist_2025}.
When used in real-time teleoperation for data collection, this delay can insert hesitant behavior, reducing the quality of trained policies. 

This paper addresses the issues of hands-only teleoperation and inference delays by proposing \ourmethod, a closed-form, provably-optimal, geometric (keypoint-based) retargeting algorithm (see \Cref{fig: front figure}). 
Unlike existing retargeting methods, \ourmethod does not use the robot's Jacobian matrix or iteratively solve optimization.
Instead, it performs pose-retargeting using geometric subproblems \cite{PadenM86,murray2017mathematical} similar to \cite{elias2025ik, elias_redundancy_2024, ostermeier_automatic_2025} that solve end-effector IK.
We define human limb vectors as unit vectors between keypoints (shoulder, elbow, wrist, etc.), then solve for robot joint angles that maximize \textit{pose similarity} by aligning these vectors with corresponding robot limb vectors (see \Cref{fig: robot and human arm diagram}).
Critically, we define pose similarity through \textit{orientation error}, as opposed to Euclidean error, which makes our method calibration-free between different human and robot sizes.

\subsubsection*{Contributions}
We make three key contributions:
\begin{itemize}
    \item We propose \ourmethod as a fast, optimal algorithm for upper body humanoid retargeting and teleoperation.

    \item We propose a safety filter that uses \ourmethod to avoid self collisions for bimanual teleoperation.

    \item We provide open-source, standalone applications to integrate \ourmethod with either MediaPipe \cite{lugaresi2019mediapipe} or a Meta Quest headset for teleoperation, and a live web demo to show computation speed (code and web demo to be released after review).
\end{itemize}
Furthermore, extensive experiments show that \ourmethod not only has high pose similarity and low computation speed, but also impacts teleoperation success rate, autonomous policy learning, and full-body humanoid retargeting speed.
We also demonstrate \ourmethod on several hardware platforms.

\subsubsection*{Paper Organization}
\Cref{sec: preliminaries} introduces notation and preliminaries.
\Cref{sec: method} then details our proposed retargeting problem and solution.
\Cref{sec: safety filter} develops a safety filter for teleoperation.
\Cref{sec: experiments} presents experiments evaluating our approach, and
\Cref{sec: demos} presents hardware demonstrations.
Finally, \Cref{sec: conclusion} provides concluding remarks and a discussion of limitations.
We also provide an extensive appendix.
\section{Related Work} \label{sec: related-works}

State-of-the-art bimanual and full-body robot teleoperation systems can be categorized into hardware puppeteering (\Cref{sec: related-works/hardware}) and software pose retargeting (\Cref{sec: related-works/software}).
We now review popular works in each category in terms of \textit{computation speed} and \textit{pose similarity} (orientation alignment between human and robot links per \Cref{sec: intro}).
Our approach uses analytical geometric IK, which we discuss last (\Cref{sec: related_works/analytical}).

\subsection{Hardware-based Teleoperation}\label{sec: related-works/hardware}

Hardware puppeteering pairs target robot arms with leader arm hardware such that a human operator directly provides joint angle commands to the robot (e.g., \method{ALOHA}~\cite{fu2024mobile} and \method{GELLO}~\cite{wu2023gello}).
This approach has high \textit{computation speed}, but low \textit{pose similarity} due to the embodiment gap of the operator directly controlling robot joints rather than the robot mimicking human arm poses. 
This issue can be solved with more human-centric leader arms functioning as exoskeleton cockpits \cite{ishiguro2020bilateral, yang2025ace, zhong2025nuexo, ben2025homie}, but such hardware impacts accessibility and applicability across a variety of robots.
We provide a software-based approach that targets similar pose similarity and computation speed to exoskeleton cockpits.

\subsection{Software-based Retargeting Teleoperation}\label{sec: related-works/software}

Software-based methods generally support more robot embodiments, but often lack \textit{computation speed} and/or \textit{pose similarity}.
Such methods usually obtain human motion from cameras or wearable devices and use inverse kinematics (IK) to retarget them to either end-effector motion or the entire arm.

\subsubsection{End-Effector Retargeting}\label{sec: related_works/iterative}

Methods in this category typically focus on an operator's 6D hand pose data (e.g., \method{DexMimicGen}~\cite{jiang_dexmimicgen_2025} and \method{OpenTeach}~\cite{iyer_open_2024}) and compute IK for the corresponding robot joint configurations with solvers such as Mink \cite{Zakka_Mink_Python_inverse_2025} and OSC \cite{khatib2003unified}; thus, they treat teleoperation as moving the two end-effectors freely in the 3D workspace. 
This approach, however, struggles when end-effector motion is limited by the arm configurations, especially for singularities at full extension that can cause reduce numerical stability and computation speed \cite{ostermeier_automatic_2025}.
By contrast, our method does not require a Jacobian and thus does not suffer numerical stability issues near singularities.
Other optimization-based IK methods avoid entering singularities poses in the first place (e.g., \method{BunnyVisionPro}~\cite{ding_bunny-visionpro_2024} and \method{Open-TeleVision}~\cite{cheng_open-television_2024} in Unitree's \method{xr\_teleoperate}~\cite{unitreerobotics_xr_teleoperate}) by using a weighted Jacobian pseudo-inverse (typically solved with Pinocchio \cite{carpentier2019pinocchio}).
However, for 7DoF humanoid arms, these methods exhibit non-cyclicity due to the redundant DoF \cite{elias_redundancy_2024}: when the robot end-effector returns to a previously visited pose, the elbow may not return to its original position.
To alleviate this, OSC \cite{khatib2003unified} allows the user to define a secondary objective in joint space (such as preserving initial joint angles) by projecting the joint error to the Jacobian nullspace.
However since we cannot directly map human elbow pose to the robot joint space without solving IK first, this approach still does not give the human operator direct control over the robot's elbow pose, further exacerbating a lack of \textit{pose similarity}.
By contrast, our method enables direct elbow control.


\subsubsection{Keypoint-Based Retargeting}\label{sec: related_works/data_driven}
To circumvent Jacobian singularities and handle complex full-body humanoid dynamics, several methods use customized optimization cost coupled with large human motion datasets (e.g., AMASS \cite{mahmood_amass_2019}) to train real-time retargeting and balance policies (e.g., \method{H2O}~\cite{he_learning_2024}, \method{CLONE}~\cite{li2025clone}, and \method{TWIST}~ \cite{ze_twist_2025}).
These methods exhibit high \textit{pose similarity}, but often still suffer large joint position errors that make dexterous tasks difficult compared to iterative methods \cite{jiang_dexmimicgen_2025}. 
Furthermore, these methods suffer high latency (around 0.7 second), some due to a combination of slow convergence of complex optimization programs and wireless communication; this can cause a teleoperator to teach a robot hesitant behaviors \cite{ze_twist_2025}.
While we do not address communication latency, our method enables high computation speed to help address the overall latency challenge. 
\ourmethod works as a drop-in replacement for General Motion Retargeting (\method{GMR}), which \method{TWIST} uses as a kinematic initial guess.
We treat each leg's hip-knee-ankle as shoulder-elbow-wrist to give \method{TWIST} similar retargeting accuracy to \method{GMR} but 1-3 orders of magnitude faster computation speed (see \Cref{subsec: exp: full body twist}).

\subsection{Analytical Geometric IK Solvers}\label{sec: related_works/analytical}

To avoid the above challenges stemming from hardware limitations, Jacobian singularities, and low computation speed, an alternative approach is to use analytical IK solvers, which typically focus on 6DoF manipulators.
Such methods have a long history \cite{pieper1969kinematics,paden1986thesis}, but have recently been improved by \method{ik-geo}~\cite{elias2025ik} and \method{EAIK}~\cite{ostermeier_automatic_2025}.
In particular, \cite{elias2025ik} provides convenient solutions to the Paden-Kahan subproblems \cite{paden1986thesis} that underlie analytical IK through closed-form, geometric decomposition.
A recent follow-up to \method{ik-geo}, \method{stereo-sew}~\cite{elias_redundancy_2024}, parameterizes the redundant joint of 7-DoF arms using elbow angles while applying subproblem decomposition for end-effector pose IK; this results in either a closed-form solution or a low-dimensional search, depending on the robot's morphology.
These approaches enable rapid computation speed, but focus on pure IK as opposed to retargeting for high pose similarity.
By contrast, \ourmethod directly addresses pose similarity while preserving high computation speed by always having a closed-form solution.
We note that considering an elbow angle for analytical IK is not new \cite{kreutz1992kinematic,asgari2025singularities,elias_redundancy_2024,salunkhe2025cuspidal}, but our SEW keypoint approach in retargeting is novel, to the best of our knowledge.
\section{Preliminaries}\label{sec: preliminaries}

We now introduce notation, model human and robot limbs, and review canonical geometric subproblems used as building blocks for our method.

\subsection{Notation}

\subsubsection*{Conventions}
We denote the real numbers as $\R$, the natural numbers as $\N$, and the space of 3-D rotation matrices as $\SO(3)$.
Scalars are in italic font, such as $x \in \R$.
Vectors and matrices are in bold font, such as $\vc{x} \in \R^n$ and $\vc{A} \in \R^{n\times m}$ with $n, m \in \N$.
The $n$-dimensional identity matrix is $\eye_n$.
An $n\times m$ array of zeros is $\zeros_{n\times m}$.
The pseudoinverse of a matrix $\vc{A}$ is $\vc{A}\pinv$.
For consistency with numbered robot joints, indexing starts at 1.
The $i^\regtext{th}$ element of a vector is $\vc{v}\arridx{i}$,  elements $i$ to $j$ are $\vc{v}\arridx{i:j}$, and the $(i,j)$th element of an array is $\vc{A}\arridx{i,j}$.
We concatenate vectors $\vc{a}, \vc{b}$ as $(\vc{a},\vc{b}) = [\vc{a}\trans,\vc{b}\trans]\trans$.
Subscripts indicate labels and superscripts indicate coordinate frames.

\subsubsection*{Common Operations}
Given a rotation axis $\axis \in \R^3$ and a rotation angle $\alpha \in \R$, we use Rodrigues' formula to create the corresponding rotation matrix:
\begin{align}\label{eq: rodrigues' formula}
    \rodrigues(\axis,\alpha) = \eye_3 + (\sin\alpha)\skewmat(\axis) + (1 -\cos\alpha)(\skewmat(\axis))^2,
\end{align}
where $\skewmat: \R^3 \to \R^{3\times 3}$ is the standard ``hat'' operator that returns a skew-symmetric matrix.
We normalize vectors to unit length as $\normalize(\vc{v}) = \vc{v} / \norm{\vc{v}}_2$.
For steps of an algorithm, we use $x \gets y$ to denote that variable $x$ has been assigned value $y$.

\subsubsection*{Coordinate Frames}
Denote the robot baselink frame as the $0^\regtext{th}$ frame.
We represent coordinate frames via a rotation matrix and translation vector with respect to the $0^\regtext{th}$ frame.
For example, frame ``f'' is $(\orthoframe{\regtext{f}},\frameorigin{\regtext{f}})$, and a vector $\vc{v}$ in that frame is $\vc{v}\frm{\regtext{f}}$.
For objects in the inertial frame, we omit the frame label when clear from context.
Given two frames $(\orthoframe{\regtext{a}},\frameorigin{\regtext{a}})$ and $(\orthoframe{\regtext{b}},\frameorigin{\regtext{b}})$, we transform from a to b in the standard way:
\begin{align}\label{eq: frame transformation}
    \vc{v}\frm{\regtext{b}} = [\orthoframe{\regtext{b}}]\trans
    (\orthoframe{\regtext{a}}\vc{v}\frm{\regtext{a}} - \frameorigin{\regtext{a}}) + \frameorigin{\regtext{b}}.
\end{align}
Numbered frames (typically indexed by $i$) refer to robot links, whereas other frames are in regular text (e.g. ``$\human$'' for the human and ``$\streaminput$'' for an 3-D keypoint input data stream).

\subsection{Human and Robot Arm Descriptions}

\begin{figure}[t]
    \centering
    \includegraphics[width=0.45\textwidth]{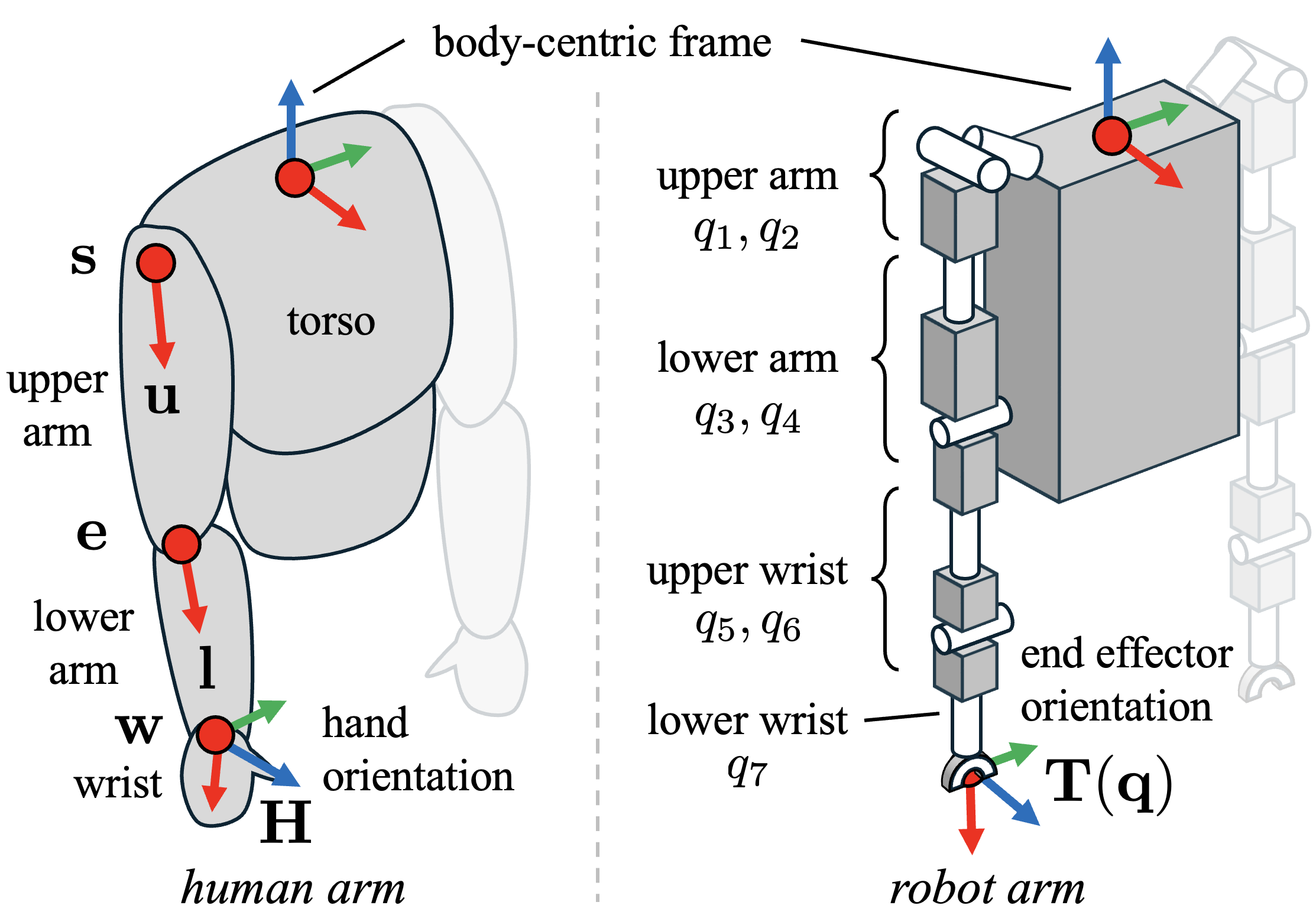}
    \caption{Comparison of human arm with keypoints and 7-DOF upper body robot arm showing links (boxes) and joints (cylinders show rotation axes).}
    \label{fig: robot and human arm diagram}
    \vspace*{-0.5cm}
\end{figure}

Our method retargets a human limb pose to a corresponding robot limb on a humanoid robot.
For ease of exposition, the majority of the paper focuses on a single limb (right arm).
The following concepts are shown in \Cref{fig: robot and human arm diagram}.

\subsubsection{Human Arm}
We consider a tuple of 3-D keypoints representing shoulder $\shoulder$, elbow $\elbow$, and wrist $\wrist$, plus a rotation matrix $\handorientation$ representing the orientation of the person's hand relative to the $0^\regtext{th}$ frame, as shown in \Cref{fig: robot and human arm diagram}.
That is, the input to our method is
    $(\shoulder,\elbow,\wrist,\handorientation) \in \R^3 \times \R^3 \times \R^3 \times \SO(3)$.
\begin{rem}[Body-Centric Frame]
Our notation implies all keypoints are in a shared $0$th \textit{body-centric frame}, which we take as an assumption going forward.
We discuss how to ensure this in \Cref{app: syncing human and robot body centric frames}.
\end{rem}

\subsubsection{Robot Arm}\label{subsubsec: robot arm preliminaries}
We consider a 7-DoF series kinematic chain of rotational actuators, where consecutive joints rotate perpendicular to each other, as shown in \Cref{fig: robot and human arm diagram}.
We define a rotation angle $\jointangle\idx{i}$ and rotation axis $\axis\frm{i}\idx{i}$ in each link’s local coordinate frame for each $i^\regtext{th}$ DoF, where $i = 1$ is the first joint, resulting in a pose vector
$\config = \mat{\jointangle\idx{1},\jointangle\idx{2},\cdots,\jointangle\idx{n}}\trans$.
Each $i^\regtext{th}$ joint is associated with a rigid link, with local coordinate frame represented in predecessor frame $i-1$ by a rotation matrix and translation vector $(\rotmat\local\frms{i-1}{i}, \transvec\local\frms{i-1}{i})$.
The orientation difference between the $(i-1)^\regtext{th}$ and $i^\regtext{th}$ local frames is the product between the known fixed local orientation difference and varying axis rotation from $\jointangle\idx{i}$:
$\rotmat\frms{i-1}{i}\actual{i} = 
    \rotmat\frms{i-1}{i}\local \rodrigues(\axis\idx{i}, \jointangle\idx{i})$.
Then the orientation difference between the $i^\regtext{th}$ and $j^\regtext{th}$ frames is given by a product of rotation matrices per standard rigid body kinematics.
To represent rotation axis $\axis\idx{i}$ in the $j^\regtext{th}$ local coordinate frame, we write
$\axis\idx{i}\frm{j} = \rotmat\frms{j}{i}(\config)\axis\frm{i}\idx{i}$.
Finally, we denote the robot's end effector orientation as
    $\toolorientation(\config) = \prod_{i=1}^{\regtext{\# DOFs}} \rotmat\frms{i-1}{i}(\config)\rotmat\lbl{align}$,
where $\rotmat\lbl{align}$ is a fixed transformation such that the end effector obeys a right-hand rule convention; in the case of a dexterous hand, ``X/Y/Z'' corresponds to index finger extended / palm normal / thumb extended, and ``X'' is the end effector pointing direction (i.e., normal to the flange).

\subsection{Canonical Geometric Subproblems}\label{subsec: paden's subproblems}

We solve retargeting by applying canonical geometric subproblems with closed-form solutions from \citet{elias2025ik}, which extend classical methods \cite{paden1986thesis,murray2017mathematical}.

\subsubsection*{Subproblem 1}\label{subsec: subproblem 1}
Consider two vectors $\vc{p}\idx{1}, \vc{p}\idx{2} \in \R^3$ and a unit vector $\vc{k} \in \R^3$, representing a rotation axis.
Rotate $\vc{p}\idx{1}$ about $\vc{k}$ to align it with $\vc{p}\idx{2}$.
That is, find the optimal angle $\theta\opt$ that minimizes $\norm{\rodrigues(\vc{k},\theta)\vc{p}\idx{1} - \vc{p}\idx{2}}$:
\begin{align}\begin{split}\label{eq: subproblem 1}
    \theta\opt \gets \subproblem{1}(\vc{p}\idx{1}, \vc{p}\idx{2}, \vc{u}) = \argmin_\theta \norm{\rodrigues(\vc{k},\theta)\vc{p}\idx{1} - \vc{p}\idx{2}}.
\end{split}\end{align}
We solve Subproblem 1 in closed form using \Cref{alg: subproblem 1} (in \Cref{app: subproblem algorithms}).

\subsubsection*{Subproblem 2}\label{subsec: subproblem 2}
Consider two vectors $\vc{p}\idx{1}, \vc{p}\idx{2} \in \R^3$ and two rotation axes $\vc{k}\idx{1},\vc{k}\idx{2} \in \R^3$.
Align $\vc{p}\idx{1}$ with $\vc{p}\idx{2}$ by simultaneously rotating $\vc{p}\idx{1}$ around $\vc{k}\idx{1}$ and rotating $\vc{p}\idx{2}$ around $\vc{k}\idx{2}$.
That is, find a pair of optimal rotation angles:
\begin{align}\begin{split}\label{eq: subproblem 2}
    \{(\theta\idx{1}\opt,\theta\idx{2}\opt)\idx{j}\}_{j=1}^2 &\gets\ \subproblem{2}(\vc{p}\idx{1}, \vc{p}\idx{2},\vc{k}\idx{1},\vc{k}\idx{2}) \\
    &= \argmin_{\theta\idx{1},\theta\idx{2}}
    \norm{\rodrigues(\vc{k}\idx{1},\theta\idx{1})\vc{p}\idx{1} - \rodrigues(\vc{k}\idx{2},\theta\idx{2})\vc{p}\idx{2}},
\end{split}\end{align}
where there may be $1$ or $2$ unique optimizers per joint angle pair.
We solve Subproblem 2 in closed form using \Cref{alg: subproblem 2} (in \Cref{app: subproblem algorithms}), which uses Paden-Kahan Subproblem 4 also presented in \Cref{app: subproblem algorithms}.

\section{Proposed Approach}\label{sec: method}

To detail our proposed approach, we first state retargeting as an orientation alignment problem (\Cref{sec: problem statement}), then present our \ourmethod algorithm (\Cref{sec: proposed method}), and finally show that our method is optimal (\Cref{sec: proving optimality}).

\subsection{Problem Statement}\label{sec: problem statement}

We seek to align an arbitrary robot arm with given human arm keypoints $(\shoulder,\elbow,\wrist,\handorientation)$ while accommodating the fact that the robot may have different link and limb lengths.
Thus, unlike other methods that focus on keypoint position error (potentially in weighted combination with orientation error \cite{araujo2025retargeting}), we focus entirely on orientation error between the human and robot upper arms, lower, arms, and wrists.
This orientation-focused approach has both theoretical and practical benefits: it enables an optimal geometric solution (\Cref{sec: proving optimality}) and applies directly to teleoperating robots with different sizes and proportions from humans (\Cref{sec: demos,sec: experiments}).

Before writing our problem, we define the human and robot upper arm, lower arm, and wrist as follows.
We convert input human keypoints to upper and lower arm unit vectors:
$\upperarm \gets \normalize(\elbow - \shoulder)$
and
$\lowerarm \gets \normalize(\wrist - \elbow)$.
For the robot, we use the $3^\regtext{rd}$ joint rotation axis $\axis\idx{3}$ as the upper arm direction, which we find applies for a variety of manipulator robots (see \Cref{sec: experiments}).
Similarly, we treat the robot's lower arm as its $5^\regtext{th}$ joint rotation axis $\axis\idx{5}$. 
Note that these human and robot vectors are generally given in two separate coordinate frames.
We propose a calibration-free procedure to sync the coordinate frames in \Cref{app: syncing human and robot body centric frames}.
For the wrist, we use the human hand orientation $\handorientation$ as a target for aligning the robot end effector orientation $\toolorientation(\config)$.
We assume the end effector mount normal is parallel to the final joint axis, and discuss the perpendicular case in \Cref{app: perpendicular wrist}.
Now we are ready to state our retargeting problem.

\begin{problem}\label{prob: retargeting problem}
We pose the retargeting problem as
\begin{align*}
    \min_{\config}\ 
        &\underbrace{\similarity(\upperarm,\rotmat\frms{0}{3}(\config)\axis\idx{3})^2}_{\regtext{upper arm}} +
        \underbrace{\similarity(\lowerarm,\rotmat\frms{0}{5}(\config)\axis\idx{5})^2}_{\regtext{lower arm}} +   
        \underbrace{\matrixsimilarity(
            \toolorientation(\config),
            \handorientation)^2}_{\regtext{wrist}}, \\
    \regtext{s.t.}\ &\config\ \regtext{obeys joint angle limits},
\end{align*}
where vector orientation error is given by cosine similarity
\begin{align}\label{eq: similarity metric}
    \similarity(\vc{u},\vc{v}) = \tfrac{1}{2} - 
        \tfrac{1}{2}\tfrac{\vc{u}\cdot\vc{v}}{\norm{\vc{u}}_2\cdot\norm{\vc{v}}_2} \in [0,1],
\end{align}
and rotation matrix orientation error is
\begin{align}\label{eq: matrix similarity metric}
    \matrixsimilarity(\rotmat\idx{1},\rotmat\idx{2}) =
        \tfrac{1}{2}
        \norm{\big(\rotmat\idx{1}\trans\rotmat\idx{2}\big)^\frac{1}{2} - \eye}\lbl{F} \in [0,1],
\end{align}
where $\norm{\cdot}\lbl{F}$ is the Frobenius norm.
\end{problem}
\noindent Note we use a minor modification of the standard chordal rotation matrix error metric \cite{hartley2013rotation}, where the square root scales the error to $[0,2]$ and makes it more linear near 0; this is not critical to our method, but provides nicer behavior for comparing errors between retargeting methods.
Next, we solve \Cref{prob: retargeting problem} by minimizing each cost term separately.

\subsection{Proposed Solution}\label{sec: proposed method}

\begin{figure}[t]
    \centering
    \includegraphics[width=\columnwidth]{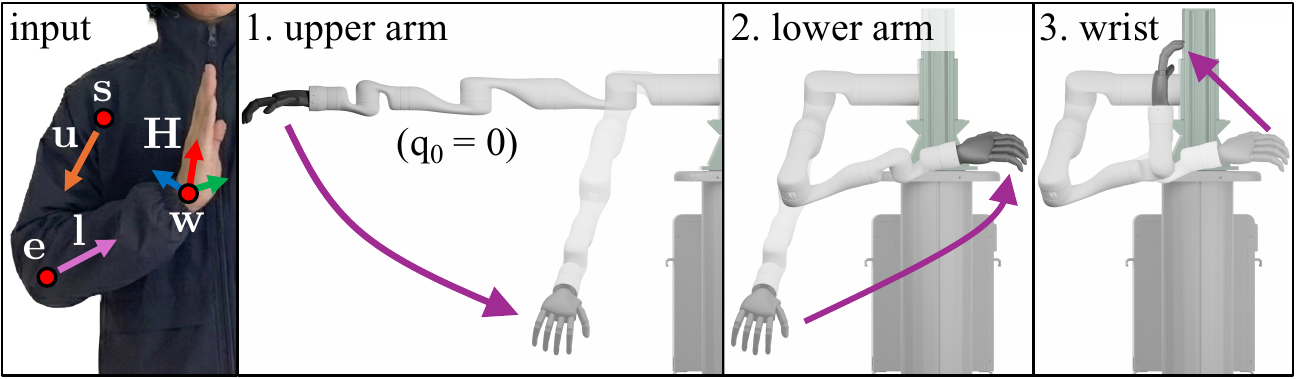}
    \caption{\ourmethod takes in human arm keypoints (left) and aligns the robot's upper arm, then lower arm, then wrist.
    }
    \vspace*{-1em}
    \label{fig: method diagram}
\end{figure}

Our closed-form geometric retargeting method operates iteratively down the arm from shoulder to elbow to wrist as shown in \Cref{fig: method diagram}, so we call it \ourmethod for ``Shoulder-Elbow-Wrist Mimic.''
Our approach is summarized in \Cref{alg: SEW-Mimic}. 
Here, we provide supporting explanation.

\subsubsection{Upper and Lower Arm Alignment}
We begin by creating the robot pose and upper and lower arm vectors (\Cref{line: init sew mimic sol}--\Cref{line: get upper and lower arms}).
Second, we align the robot's $\axis\idx{3}$ axis with the human upper arm $\upperarm$ (\Cref{line: align upper arm}).
Third, we similarly align $\axis\idx{5}$ with the human lower arm $\lowerarm$ (\Cref{line: align lower arm}).
For both of these alignment steps, we solve for two joint angles at a time using \Cref{alg: align axis}.

\subsubsection{Aligning Joint Axes to Given Vectors}
\Cref{alg: align axis} aligns the $i$th robot joint axis $\axis\idx{i}$ to a given vector $\vc{v}$ (i.e., minimize $\similarity(\axis\idx{i},\vc{v})$) by finding optimal angles for joints $i-1$ and $i-2$.
We operate in the $i-2$ coordinate frame (\Cref{line: start map to i-2 frame}--\Cref{line: finish map to i-2 frame}) map to that frame), and uses Subproblem 2 (\Cref{alg: subproblem 2}) to find the predecessor joint angles (\Cref{line: use subproblem 2}).
We filter out solutions that do not obey joint angle limits (\Cref{line: bound joint angles}); note, this does \textit{not} consider self-collision, which we handle with a safety filter in \Cref{sec: safety filter}.
Finally, we choose the closest solution to the robot's current configuration (\Cref{line: find closest joint angle}).

\subsubsection{Wrist Alignment}
We seek to have the tool orientation $\toolorientation(\config)$ match the desired hand orientation by solving for $(\jointangle\idx{5},\jointangle\idx{6},\jointangle\idx{7})$.
We assume the wrist has the end effector mount pointing parallel to the 7th joint axis per \Cref{subsubsec: robot arm preliminaries}.
Then, we apply Subproblem 2 to find $\jointangle\idx{5}$ and $\jointangle\idx{6}$ that align $\axis\idx{7}$ with the human hand pointing direction (\Cref{line: wrist a subproblem 2}), and finally Subproblem 1 to solve for $\jointangle\idx{7}$.


\begin{algorithm}[t]
\DontPrintSemicolon 
\caption{$\config \gets \ourmethod(\config\init,\shoulder,\elbow,\wrist,\handorientation)$}
\label{alg: SEW-Mimic}
\setstretch{1.1}
\KwInput{Initial robot pose $\config\init$, human shoulder position $\shoulder$, human elbow position $\elbow$, wrist position $\wrist$, and human hand orientation $\handorientation$}

// Sync human and robot frames per \Cref{app: syncing human and robot body centric frames} \newline
$\config \gets \config\init$  // Initialize solution \label{line: init sew mimic sol} \;

// Get upper and lower arm pointing directions \newline
$\upperarm \gets \normalize(\elbow - \shoulder)$ and
$\lowerarm \gets \normalize(\wrist - \elbow)$ \label{line: get upper and lower arms}\;

// Find $(\jointangle\idx{1},\jointangle\idx{2})$ by aligning $\axis\idx{3}$ (proxy for robot upper arm) with $\upperarm$ using Subproblem 2 per \Cref{alg: align axis} \newline
$\config\arridx{1:2} \gets \alignaxis(3,\config,\upperarm)$ \label{line: align upper arm}

// Find $(\jointangle\idx{3},\jointangle\idx{4})$ by aligning $\axis\idx{5}$ (proxy for robot lower arm) with $\lowerarm$ using Subproblem 2 per \Cref{alg: align axis} \newline
$\config\arridx{3:4} \gets \alignaxis(5,\config,\lowerarm)$ \label{line: align lower arm}

// Align wrist to get $\config\arridx{5:7}$ \newline
$\config\arridx{5:7} \gets \alignwrist(\config,\handorientation)$ // See \Cref{alg: align wrist} \label{line: align wrist} \;

\KwReturn{Retargeted robot pose $\config$}
\end{algorithm}

\begin{algorithm}[t]
\DontPrintSemicolon 
\caption{Align Joint Axis $i$ to a Given Vector \\
$(\jointangle\idx{i-2}\opt,\jointangle\idx{i-1}\opt) \gets \alignaxis(i,\config\init,\vc{v})$ \\
Align the $i$th joint axis $\axis\idx{i}$ with a vector $\vc{v}$ by solving for the axis' predecessor joint angles $\jointangle\idx{i-2}$ and $\jointangle\idx{i-1}$}
\label{alg: align axis}
\setstretch{1.1}

\KwInput{joint index $i$, current configuration $\config\init$, and vector to align with $\vc{v}$}

$\vc{v}\frm{i-2} \gets [\rotmat\frms{0}{i-2}(\config\init)]\trans \vc{v}$ // Put vector into $i-2$ frame  \label{line: start map to i-2 frame} \;

$\axis\idx{i-2}\frm{i} \gets \rotmat\frms{i-2}{i}(\config\init) \axis\idx{i}$ // Put joint axes in $i-2$ frame \;

$\axis\idx{i-2}\frm{i-1} \gets \rotmat\frms{i-1}{i}(\config\init) \axis\idx{i-1}$ \label{line: finish map to i-2 frame} \;

// SP2 aligns $\axis\idx{i}$ to $\vc{v}$ and returns up to 2 pairs of sols., indexed here by $j$ \;

$\{(\jointangle\idx{i-2},\jointangle\idx{i-1})_j\} \gets 
        \subproblem{2}(\vc{v}\frm{i-2},\axis\idx{i}\frm{i-2},-\axis\idx{i-2},\axis\idx{i-1}\frm{i-2})$ \label{line: use subproblem 2}\;

$\{(\jointangle\idx{i-2},\jointangle\idx{i-1})_j\} \gets \boundjointangles(\{(\jointangle\idx{i-2},\jointangle\idx{i-1})_j\})$ \label{line: bound joint angles}\;

// Update pose with closest angles to init. pose \newline
$(\jointangle\idx{i-2}\opt,\jointangle\idx{i-1}\opt) \gets \argmin_{a,b} \abs{\config\init\arridx{i-2} - a} + \abs{\config\init\arridx{i-1} - b}$ 
s.t. $(a,b) \in \{(\jointangle\idx{i-2},\jointangle\idx{i-1})_j\}$ \label{line: find closest joint angle} \;

\KwReturn{Joint angle solution $(\jointangle\idx{i-2}\opt,\jointangle\idx{i-1}\opt)$}
\end{algorithm}

\begin{algorithm}
\DontPrintSemicolon 
\caption{Align Robot Wrist to Hand Orientation \\
$(\jointangle\idx{5},\jointangle\idx{6},\jointangle\idx{7}) \gets \alignwrist(\config\init,\handorientation)$}
\label{alg: align wrist}
\setstretch{1.1}

 \KwInput{Init. pose $\config\init$, hand orientation $\handorientation$}

 // Require: $\rotmat\frms{7}{\toolorientation}\local$ (EE orientation in $7$th joint frame) \;

$\rotmat\frms{0}{7}\lbl{des} \gets \handorientation [\rotmat\lbl{align}]\trans[\rotmat\frms{7}{\toolorientation}\local]\trans$ // Desired orientation of 7th joint\;

    // Align $\axis\idx{7}$ with human hand using Subproblem 2 \newline
    $(\jointangle\idx{5},\jointangle\idx{6}) \gets \alignaxis(7,\config,\rotmat\frms{0}{7}\lbl{des}\arridx{:,1})$  \label{line: wrist a subproblem 2}
    
    // Put $\axis\idx{6}$ and desired direction in 7th joint frame \newline
    $\vc{u}\idx{6}\frm{7} \gets \rotmat\frms{7}{0} \rotmat\frms{0}{6}(\config\init) \axis\idx{6}$ \newline    
    $\axis\idx{6}\frm{7} \gets [\rotmat\frms{6}{7}\local]\trans \axis\idx{6}$ \;
    
    // Get $\jointangle\idx{7}$ using Subproblem 1 and enforce joint limits \newline
    $\jointangle\idx{7} \gets \boundjointangles(
    \subproblem{1}(\axis\idx{6}\frm{7}, \vc{u}\idx{6}\frm{7} , -\axis\idx{7}))$ \label{line: wrist a subproblem 1}

\KwReturn{Final joint angles $(\jointangle\idx{5},\jointangle\idx{6},\jointangle\idx{7})$}
\end{algorithm}

\subsection{Optimality of Proposed Algorithm}\label{sec: proving optimality}

To conclude the section, we confirm that \Cref{alg: SEW-Mimic} returns an optimal solution to \Cref{prob: retargeting problem}.

\begin{restatable}{prop}{myPropLabel}
\label{prop: sew-mimic is optimal}
    Consider a robot arm with consecutive perpendicular joint axes, baselink mounted in a humanoid configuration as shown in \Cref{fig: robot and human arm diagram}(b), and no joint angle limits or self-collisions.
    Suppose the robot is at an initial configuration $\config\init$.
    Suppose we are given keypoints of the human shoulder $\shoulder$, elbow $\elbow$, and wrist $\wrist$, and a hand orientation $\handorientation$.
    Using \Cref{alg: SEW-Mimic}, compute
    $\config \gets \ourmethod(\config\init,\shoulder,\elbow,\wrist,\handorientation)$.
    Then $\config$ is a global optimizer of \Cref{prob: retargeting problem}.
\end{restatable}
\noindent See \Cref{app: optimality proof} for the proof.
Note that the solution is optimal \textit{in the body-centric frame}, meaning that human input torso motion must be canceled out (as per \Cref{app: syncing human and robot body centric frames}) before passing keypoints to \ourmethod.
Also note, our claim is in the absence of joint angle limits, safety filtering for self-collision avoidance, or joint controller tracking error.
That said, our experiments (\Cref{sec: experiments}) indicate a variety of benefits from this optimality.

\section{Safety Filter for Self-Collisions}\label{sec: safety filter}

To mitigate the risk of self-collision in bimanual teleoperation, we create a safety filter that takes advantage of our SEW representation.
We design it to remove only the harmful component of a user command, preserving tangential motion to collisions so the robot can be operated along safe directions.
While the filter does not provide formal guarantees (left to future work), we find it is fast ($\sim$250 Hz) and reduces self-collisions (see \Cref{subsec: exp: safety filter}).

\begin{rem}
    The key idea is that our fast closed-form solution leaves computation overhead for downstream processes such as safety filtering by serving as a drop-in, fast humanoid IK.
\end{rem}
\noindent Next, we summarize the method, then detail specific parts.

\subsubsection*{Algorithm Overview}
The safety filter is summarized in \Cref{alg: safety filter}.
Given an initial bimanual pose and a potentially-unsafe desired pose,
we first generate capsules from the robot's SEW keypoints in that configuration (\Cref{line: make capsules}).
We then adjust the keypoints to safety by applying extended position-based dynamics (XPBD)~\cite{macklin2016xpbd} (\Cref{line: safety filter xpbd update}).
Finally, we use \ourmethod to map the adjusted keypoints back to the robot's joint angles (\Cref{line: safety filter start recover config}--\Cref{line: safety filter end recover config}), where $\recovertoolori$ (\cref{alg: recover tool} in \Cref{app: safety filter details}) obtains the desired robot end-effector orientation.
In this context, \ourmethod functions as a fast IK solver.
We provide more details for each of these steps below.

\begin{algorithm}[ht]
\DontPrintSemicolon 
\caption{SEW Safety Filter for Self-Collisions \\
$\config\safe \gets \safetyfilter(\config\init,\config\des)$}
\label{alg: safety filter}
\setstretch{1.1}

\KwInput{Current pose $\config\init$, desired pose $\config\des$}

$\robotkeypoints \gets \FK(\config\des)$ // get keypoints of desired pose






$\collisionvolumes \gets \makecapsules(\robotkeypoints)$ 
// init. collision volumes
\label{line: make capsules}

// Init. XPBD Lagrange multipliers \newline
$\{\lambda\idx{i,j}\} \gets 0$ // one per $(i,j)$th collision pair in $\collisionvolumes\times\collisionvolumes$


\For{$k = 1,\cdots,n\lbl{iter}$ // \# of iterations}{

    // Nudge capsules w/ \Cref{alg: XPBD update} in \Cref{app: safety filter details}\newline
    $(\collisionvolumes,\{\lambda\idx{i,j}\}) \gets \XPBDupdate(\collisionvolumes,\{\lambda\idx{i,j}\})$ \label{line: safety filter xpbd update}


    \If{capsules in $\collisionvolumes$ are collision-free}{
        // Recover SEW keypoints from capsules\newline
        $(\shoulder\lside,\elbow\lside,\wrist\lside,\tool\lside,     
        \shoulder\rside,\elbow\rside,\wrist\rside,\tool\rside) \gets \method{GetKPs}(\collisionvolumes)$ \label{line: safety filter recover keypoints}

        // Recover tool orientations \label{line: safety filter start recover config}\newline
        $\handorientation\lside \gets 
        \recovertoolori(\toolorientation\lside(\config\init), \tool\lside)$\newline
        $\handorientation\rside \gets
        \recovertoolori(\toolorientation\rside(\config\init), \tool\rside)$

        // Recover left and right arm configurations \newline
        $\config\lside \gets \ourmethod(\shoulder\lside,\elbow\lside,\wrist\lside,\handorientation\lside)$ \newline
        $\config\rside \gets \ourmethod(\shoulder\rside,\elbow\rside,\wrist\rside,\handorientation\rside)$ \label{line: safety filter end recover config}
    
        \KwReturn{Safe pose $\config\safe \gets (\config\lside,\config\rside)$}
        }

}

\KwReturn{Original pose $\config\init$ if no safe pose found}
\end{algorithm}

\subsubsection*{Bimanual Notation}\label{subsec: bimanual notation}
We consider a left/right arm pair with pose $\config = (\config\lside,\config\rside)$ where $\config\lside$ is the left arm pose and $\config\rside$ right; we redefine $\config$ here to simplify exposition.
We define a tuple of \textit{robot keypoints}:
    $\robotkeypoints =
            (\shoulder\lside,\elbow\lside,\wrist\lside,\tool\lside,     
            \shoulder\rside,\elbow\rside,\wrist\rside,\tool\rside)
            \gets \FK(\config)$,
where $\FK$ denotes standard rigid body forward kinematics, $\shoulder\lside \in \R^3$ denotes robot left shoulder location (typically the 1st joint actuator location), and similarly $\elbow\lside$ for left elbow (4th joint actuator location), $\wrist\lside$ for left wrist (6th joint actuator location), $\tool\lside$ for left tool tip, and $\toolorientation\lside$ for left tool orientation; right side quantities are labeled ``$\rightside$.'' 

\subsubsection*{Capsule Collision Volumes}
We construct capsules (i.e., spheres swept along line segments) overapproximating the robot's links for collision checking, as shown in \Cref{fig: RBY1 safety filter capsules}.
A capsule $\capsule$ from $\vc{p}\idx{1}$ to $\vc{p}\idx{2}$ with radius $\radius$ is
\begin{align}\begin{split}
    \capsule &\gets \capsulefunc(\vc{p}\idx{1},\vc{p}\idx{2},\radius) \\
        &=
        \big\{\vc{x} \in \R^3 \mid\,  
                \norm{\vc{x} - t\vc{p}\idx{1} - (1-t)\vc{p}\idx{2}}_2 \leq \radius,\, 
                t \in [0,1]
        \big\}
\end{split}\end{align}
Thus, we overapproximate the robot's left upper arm with
$\upperarmset\lside \gets \capsulefunc(\shoulder\lside,\elbow\lside,\radius\lbl{upper})$
where $\radius\lbl{upper}$ is chosen large enough.
We similarly construct a capsule $\lowerarmset\lside$ to contain the lower arm and a capsule $\handset\lside$ to contain the robot wrist and end effector, and repeat for the right side.
We also create a capsule $\torsoset$ for the robot's torso.
To simplify notation, we gather the capsules into a set of collision volumes:
\begin{align*}\begin{split}
    \collisionvolumes \gets \makecapsules(\robotkeypoints)
    = \{\torsoset,\upperarmset\lside,\upperarmset\rside,\lowerarmset\lside,\lowerarmset\rside,\handset\lside,\handset\rside\}.
\end{split}\end{align*}
A concrete example of $\makecapsules$ is in \Cref{app: safety filter evaluation details}.



\subsubsection*{XPBD Iteration}
The filter uses XPBD~\cite{macklin2016xpbd} to push capsules out of collision and keypoints into obeying kinematic constraints, which we detail in \Cref{alg: XPBD update} in \Cref{app: safety filter details} and summarieze here.
Our XPBD implementation defines a Lagrange multiplier $\lambda\idx{i,j}$ per collision volume pair, computes a collision gradient by accumulating over contact normals of each collision volumes pair, then uses $\lambda\idx{i,j}$ with user-specified weights of each volume to move each keypoint.
Finally, we enforce kinematic constraints by projecting each link back to original length (see \Cref{alg: XPBD length projection} in \Cref{app: safety filter details}).

\subsubsection*{Continuous Time Collision Checking}
In practice, if the initial and desired pose are far apart, then the safety filter can find a safe pose that requires one robot limb to ``jump'' through another, which is unsafe.
We compensate for this by preprocessing the collision volumes (i.e., modifying \Cref{line: make capsules}) to find the first collision between the initial and desired poses, as detailed in \Cref{app: cont time coll check}.

\begin{figure}[t]
    \centering
    \includegraphics[width=\columnwidth]{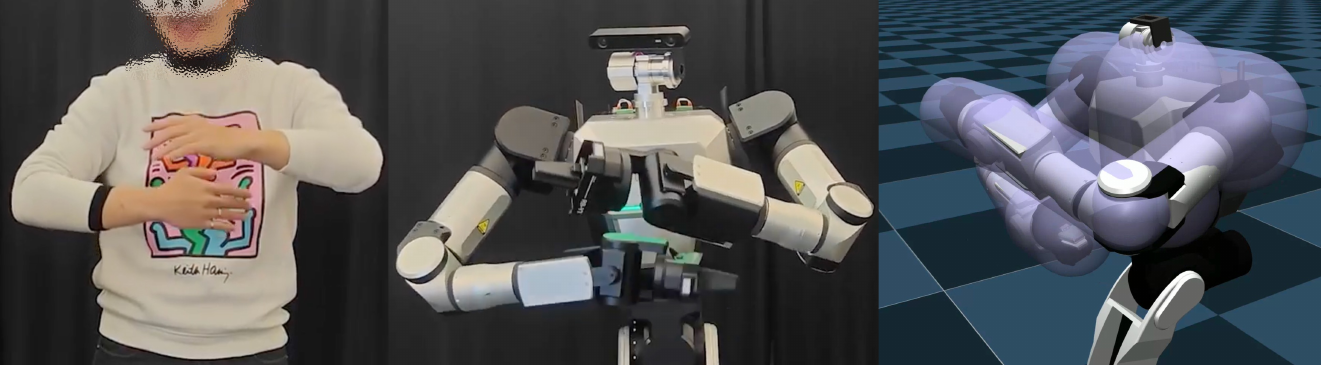}
    \caption{Our safety filter uses capsules (right) to compute and avoid self collision, as shown on Rainbow RBY1 hardware.}
    \label{fig: RBY1 safety filter capsules}
    \vspace*{-1em}
\end{figure}
\section{Experiments}\label{sec: experiments}

We now evaluate \ourmethod experimentally to answer the following questions:
\begin{outline}
    \1 (\Cref{subsec: exp: retargeting})
    How does our method compare against other retargeting approaches in terms of speed and accuracy?

    \1 (\Cref{subsec: exp: user study})
    How does our method compare in ease-of-use for teleoperation?

    \1 (\Cref{subsec: exp: safety filter})
    How effective is our safety filter?

    \1 (\Cref{subsec: exp: policy learning})
    Is our SEW representation useful for training generative robot motion policies?    

    \1 (\Cref{subsec: exp: full body twist})
    Is our SEW representation useful for full-body humanoid retargeting?    
\end{outline}

\subsection{Retargeting Performance}\label{subsec: exp: retargeting}

We investigate if \ourmethod has lower orientation error (as defined in \Cref{prob: retargeting problem}), absolute joint position (L2) error, and computation speed than baselines \cite{araujo2025retargeting} \cite{unitreerobotics_xr_teleoperate}.

\subsubsection{Experiment Design}
We evaluate \ourmethod per \Cref{alg: SEW-Mimic} and the baselines \cite{araujo2025retargeting} \cite{unitreerobotics_xr_teleoperate} for retargeting both left and right arms of the Ubisoft LAFAN1 \cite{harvey2020robust} dataset, processed to solely contain upper body motion.

\subsubsection{Results and Discussion}
The results are summarized in \Cref{fig: retargeting performance result}.
A Kruskal-Wallis test was conducted to evaluate the effect of retargeting solver on alignment error and inference time.
We see a statistically significant effect of solver on both alignment error ($H = 215864.09$, $p < 0.001$) and inference time ($H = 277254.30$, $p < 0.001$).
Dunn's post-hoc comparisons with Bonferroni correction confirm that \ourmethod achieves significantly lower alignment error ($\text{Median} = 1.57 \times 10^{-13}$, $\text{IQR} = [1.06 \times 10^{-13}, 1.66 \times 10^{-5}]$) and inference time ($\text{Median} = 0.587$ ms, $\text{IQR} = [0.573, 0.612]$ ms) than both \method{GMR} \cite{araujo2025retargeting} and \method{xr\_teleoperate} \cite{unitreerobotics_xr_teleoperate} ($p < 0.001$ for all pairwise comparisons).
Altogether, as expected given its closed-form and optimality, \ourmethod achieves greater accuracy and faster inference than the baselines, with alignment errors comparable to floating-point numerical precision and 2-3 orders of magnitude lower inference time. 
Note, \ourmethod does not beat \method{GMR} on its own custom retargeting objective, as expected  (see \Cref{app: retargeting exp details}); the point of this result is to confirm our optimality proof.

\begin{figure}[t]
    \centering
    \begin{subfigure}{0.99\columnwidth}
    \includegraphics[width=0.99\columnwidth]{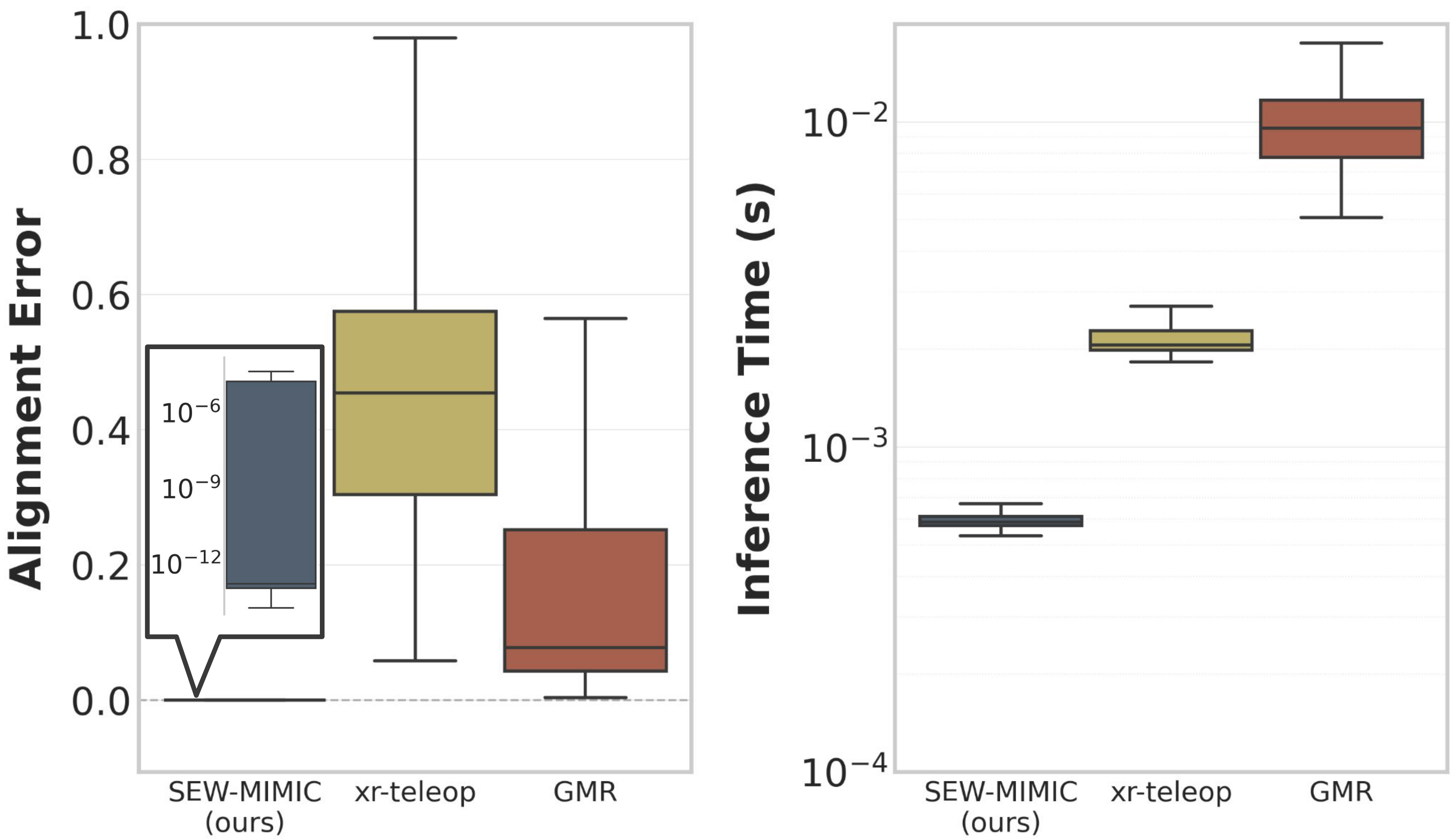}
    \end{subfigure}
    
    \caption{Retargeting alignment error and inference time of \ourmethod vs. baselines.}
    \label{fig: retargeting performance result}
    \vspace*{-0.5em}
\end{figure}

\subsection{Pilot User Study}\label{subsec: exp: user study}

We investigate if \ourmethod enables users to complete bimanual manipulation tasks more successfully than end-effector control with \basemink \cite{Zakka_Mink_Python_inverse_2025}. 

\subsubsection{Experiment Design}
We arrange a $2 \times 3 $ factorial design repeated measures (within subjects) user study ($N=8$) with the teleoperation solver and specific manipulation tasks as independent variables.
Per user and task, we measure total attempts, successes, and failures. 
Users teleoperate a Dual Kinova Gen3 robot in Robosuite \cite{zhu2020robosuite} using \ourmethod and \basemink on three tasks (see \Cref{app: user study details}, \Cref{fig: task environment showcase}):
\begin{outline}
    \1 \textit{Cabinet}:
    Move a box from a table into a raised cabinet.

    \1 \textit{Glass Gap}:
    Pick a box from between two rows of wine glasses and place it to the right (\Cref{fig: glass gap sew vs mink}).

    \1 \textit{Handover}: Move a box from start to goal, where the start and goal are far apart, necessitating a bimanual handover.
\end{outline}
The order of tasks and teleoperation methods is randomized per user.
Users wear a Meta Quest3 VR headset to which we stream binocular vision.
Users are given 5 minutes per task and controller, and are instructed to achieve as many successes as possible in the allotted time.
Further details are in \Cref{app: user study details}.

\subsubsection{Results and Discussion}

An Aligned Rank Transform (ART) ANOVA was conducted to evaluate the effect of task type and teleoperation solver on total attempts, total successes, and total failures.
We see statistically significant interaction between teleoperation methods \ourmethod and \basemink across number of successes ($p<0.01$), and number of failures ($p<0.01$).
Additionally, task had a significant effect ($p<0.05$) across all comparisons.
That said, post-hoc simple main effects analysis did not indicate statistically significant interaction when comparing the same task across different controllers for any dependent variable. 
Note our raw data are presented in \Cref{tab: user study data} in the appendix.

The statistically significant interaction between methods across all dependent variables suggests that \ourmethod has an effect on user performance.
The significant effect of task on all dependent variables indicates our tasks are all different, supporting our evaluation procedure.
A considerable limitation was the small sample size, which we suspect is why there is no significant effect of method on any one task, only in aggregate across all tasks; further study is needed on how \ourmethod affects user performance.
Qualitatively, we noticed an interesting failure mode for \basemink: given the fixed scaling factor between user and robot end effector displacement, users with shorter arms were unable to reach the box and thus could not succeed.
Finally, the raw data are promising in an informal sense that warrants further investigation: \ourmethod has 125 total successes summed across all tasks and users, versus 68 for \basemink.

\begin{figure}[t]
    \centering
    \includegraphics[width=0.99\columnwidth]{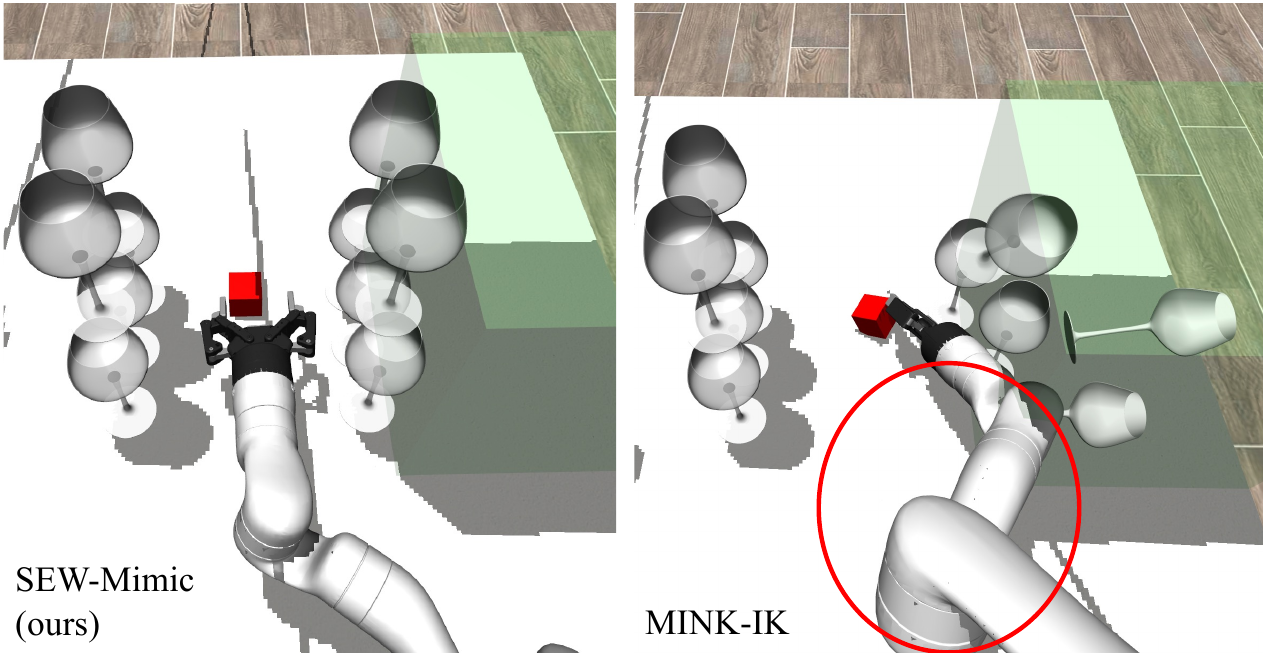}
    \caption{The Glass Gap task requires moving a red box from between wine glasses to a large green goal on the left.
    We show an example pose from our user study showing \ourmethod on the left vs. \basemink on the right, which shows how a lack of explicit elbow control can cause task failure from unexpected joint self-motion.}
    \label{fig: glass gap sew vs mink}
    \vspace*{-0.75em}
\end{figure}

\subsection{Safety Filter Ablation Study}\label{subsec: exp: safety filter}

We now investigate how our safety filter influences operation of \ourmethod near self-collisions.

\subsubsection{Experiment Design}
To create near self-collision poses, we record a continuous ``rolling punch'' motion of both human arms completing 10 full rotations in 10 seconds in front of the chest (\cref{fig: RBY1 safety filter capsules}).  
We use \ourmethod to retarget this motion to the RB-Y1 robot and measure the alignment error, number of self collisions, and computation speed with and without the proposed safety filter.
Further details in \Cref{app: safety filter evaluation details}.

\subsubsection{Results and Discussion}

    
See \Cref{fig: safety filter abolition} in \Cref{app: safety filter evaluation details}
\ourmethod without safety filter has zero alignment error, about 0.8 ms (1250 Hz) computation time for both arms, and self-collision in nearly half of the retargeted poses.
With the safety filter, we have on average higher alignment error of 0.019, longer computation time of about 4.0 ms (250 Hz), but much lower self-collision instances in only around 1.3\% of the retargeted poses.
Thus we conclude that our safety filter can reduce self-collision instances with some compromise in pose similarity and computation speed.
This tradeoff between accuracy and safety can be adjusted by each link's capsule size; we leave finding optimal tradeoffs to future work.

\subsection{Policy Learning}\label{subsec: exp: policy learning}

We now explore if data collected with \ourmethod results in a policy that has higher success rate and lower task completion time than data collected with end-effector-only motion and \basemink \cite{Zakka_Mink_Python_inverse_2025}.

\subsubsection{Experiment Design}
To test the hypothesis, we collect two matched demonstration datasets on the Glass Gap task (\Cref{fig: glass gap sew vs mink}): one using \basemink and one using \ourmethod.
Each dataset contains 50 demonstrations.
Across demonstrations, the box is initialized at random locations between the wine glasses.
To control for demonstration quality, both datasets are collected by the same expert tele-operator.
We then train Diffusion Policy \cite{chi2025diffusion,reuss2023diffusion} separately on each dataset. For each dataset, we run 3 training seeds for 500 epochs and select the best-performing checkpoint for final evaluation. During evaluation, we sample 250 initial conditions and compute the task success rate, reporting results across the three seeds.

\subsubsection{Results and Discussion}
The policy trained on demonstrations collected by \ourmethod achieves a substantially higher success rate (approximately $3\times$) than the policy trained on \basemink demonstrations, while achieving shorter task completion time (see \Cref{fig: glass_gap_eval} in the appendix).
We suspect the performance gap is related to demonstration smoothness, which we investigate in \Cref{app: policy learning details}.
In particular, \ourmethod data shows smoother joint angle trajectories and lower desired joint velocities (see \Cref{fig:per_joint_comparison}).
This is associated with lower action prediction error (see \Cref{fig:policy_learning_action_prediction_error}), despite shorter teleoperation horizons (see \Cref{fig:demo_collection_time}). Establishing whether these factors causally explain the gap requires further investigation and is deferred to future work.

\subsection{Full-Body Retargeting}\label{subsec: exp: full body twist}

Finally, we explore if \ourmethod improves a pretrained full-body humanoid retargeting policy's computational efficiency without negatively impacting accuracy.

\subsubsection{Experiment Design}
We apply \ourmethod as a drop-in replacement for the \method{GMR} retargeter \cite{araujo2025retargeting} in the \method{TWIST} full-body teleoperation framework \cite{ze_twist_2025} by treating each leg's hip-knee-ankle as shoulder-elbow-wrist.
\method{GMR} generates a kinematic initial guess of joint angles, which \method{TWIST} converts into a dynamically feasible pose (passed to a PD controller for a Unitree G1 humanoid in MuJoCo).
We do not retrain or finetune \method{TWIST} (which is trained on data generated by \method{GMR}).
We measure three metrics for \ourmethod and \method{GMR}: inference time, kinematic-to-dynamic ``kin2dyn'' joint angle error between the input and output of \method{TWIST}, the PD controller joint angle tracking error.
Further details are in \Cref{app: exp: full body twist details}.

\subsubsection{Results and Discussion}
\ourmethod performs 1-3 orders of magnitude faster than \method{GMR}, with no increase in either kin2dyn error or tracking error, per \Cref{fig: sew-twist errors}.
This is a remarkable result, because \ourmethod optimizes for a different metric than \method{GMR}, and is not designed for full-body retargeting, yet operates as a satisfactory drop-in accelerator for \method{TWIST}.

\begin{figure}[t]
    \centering
    \includegraphics[width=0.99\linewidth]{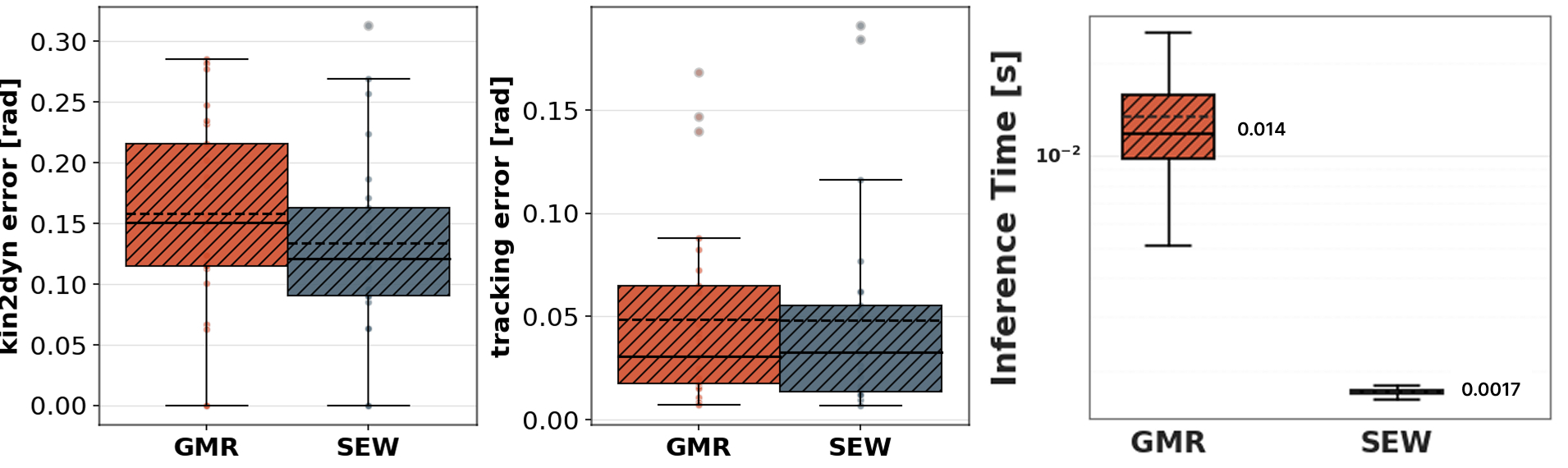}
    \caption{
    Kinematic-to-dynamic error, controller tracking error, and inference time when using \ourmethod and \method{GMR} in \method{TWIST}.
    } 
    \label{fig: sew-twist errors}
    \vspace*{-1em}
\end{figure}

\section{Demonstrations}\label{sec: demos}

We demonstrate \ourmethod on a Rainbow Robotics RB-Y1 and Kinova Gen3 dual arm systems in the video supplement.
Qualitatively, our method enables fast robot motion, avoids self-collision, and operates smoothly near singularities.
Hardware setup details are explained in \Cref{app: hardware setup}.
We test on motions such as snatching a wooden block, rolling punches (see \Cref{subsec: exp: safety filter}), crossing arms, and reaching full extension.
\section{Conclusion and Limitations}\label{sec: conclusion}

This paper presents \ourmethod, a closed-form geometric retargeting method for upper body humanoid robots.
The method optimally maps human to robot pose with respect to limb orientations, and is used to implement a fast safety filter for self collisions.
Experiments show that \ourmethod is much faster and more accurate than other retargeting baselines, has a significant effect on task success in a pilot user study, has potential for future policy learning, and improves full-body humanoid retargeting speed.

\subsubsection*{Limitations}
As is, \ourmethod is kinematic and constrained to humanoid morphologies (inapplicable for standard mounting of tabletop manipulators).
Our optimality proof ignores joint angle limits and self collisions; proof is potentially impossible if these factors are included.
Our method's accuracy also depends on human keypoint detection accuracy, so is susceptible to perception errors (though these may be mitigated by filtering and high computation speed).
The proposed safety filter, while fast and effective in practice, also lacks a correctness proof; one solution could be incorporating recent provably-correct collision avoidance methods \cite{schepp2022sara,michaux2024sparrows,wei2025collision,michaux2025can}.
Another major limitation (seen in demonstrations) is remaining latency due to wireless communication and tracking control.
This may be mitigated by incorporating feedforward velocity, which puts more pressure on the safety filter to correct future unsafe poses, versus just current poses.
Finally, a larger and more detailed user study is needed to elucidate how and why \ourmethod may increase teleoperation performance.

\bibliography{references}

\clearpage

\appendix
\subsection{Canonical Geometric Subproblem Algorithms}\label{app: subproblem algorithms}

Here we present closed-form geometric methods (see \Cref{alg: subproblem 1,alg: subproblem 2,alg: subproblem 4}) to solve Subproblems 1, 2, and 4 following the method of \cite{elias2025ik}.
First, we present Subproblem 4, which is a dependency of \Cref{alg: subproblem 2} to solve Subproblem 2.
The subproblems themselves are illustrated by example in \Cref{fig: subproblem examples}.

\begin{figure}[ht]
    \centering
    \includegraphics[width=0.99\columnwidth]{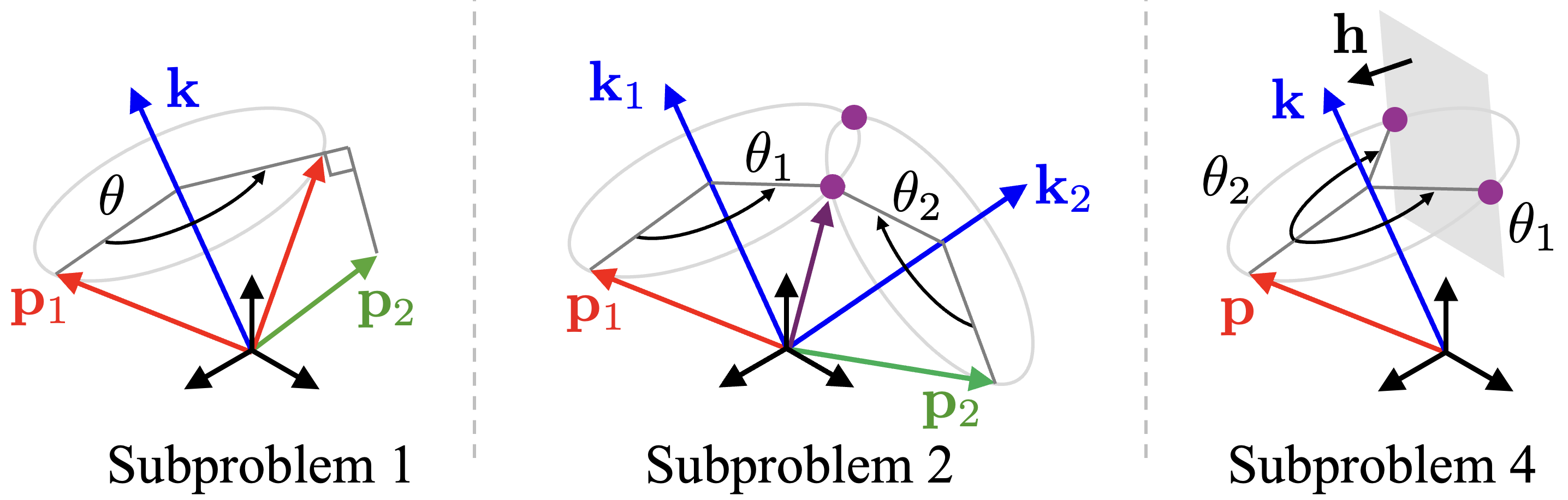}
    \caption{Examples of Subproblems 1, 2, and 4.
    For Subproblems 2 and 4, the two solution cases are shown.
    Adapted from \cite{elias2025ik}.}
    \label{fig: subproblem examples}
\end{figure}

\subsubsection*{Subproblem 4}\label{subsec: subproblem 4}
Consider a vector $\vc{p}$, a unit vector $\vc{k}$ defining a rotation axis, and a unit vector $\vc{h}$ defining a plane offset from the origin by a distance $d$.
Rotate $\vc{p}$ about $\vc{k}$ by an angle $\theta$ such that the vector $\rodrigues(\vc{k},\theta)\vc{p}$ is either on, or as close as possible to, the plane:
\begin{align}\begin{split}\label{eq: subproblem 4}
    \{\theta\opt\idx{i}\} \gets\ \subproblem{4}(\vc{p},\vc{h},\vc{k},d) = \argmin_{\theta} \abs{\vc{h}\trans\rodrigues(\vc{k},\theta)\vc{p} - d},
\end{split}\end{align}
which returns either $1$ or $2$ optimizers depending on whether or not $\vc{p}$ can be rotated to touch the plane.
We solve Subproblem 4 in closed form using \Cref{alg: subproblem 4} (in \Cref{app: subproblem algorithms}).

\begin{algorithm}[h]
\DontPrintSemicolon 
\caption{Subproblem 1 \\
$\theta\opt \gets \subproblem{1}(\vc{p}\idx{1}, \vc{p}\idx{2}, \vc{k})$ \\
Find an angle $\theta\opt$ that aligns $\vc{p}\idx{1}$ with $\vc{p}\idx{2}$ by rotating $\vc{p}\idx{1}$ about $\vc{k}$}
\label{alg: subproblem 1}
\setstretch{1.1}

\KwInput{$\vc{p}\idx{1}, \vc{p}\idx{2}, \vc{k}$}

// Project vectors onto plane perpendicular to $\vc{k}$ and find rotation angle in that plane

$\hat{\vc{p}}\idx{1} \gets \normalize(\vc{p}\idx{1} - (\vc{p}\idx{1}\cdot\vc{k})\vc{k})$

$\hat{\vc{p}}\idx{2} \gets \normalize(\vc{p}\idx{2} -(\vc{p}\idx{2}\cdot\vc{k})\vc{k})$

$\theta\opt \gets 2 \cdot \atantwo(
    \norm{\hat{\vc{p}}\idx{1} - \hat{\vc{p}}\idx{2}},
    \norm{\hat{\vc{p}}\idx{1} + \hat{\vc{p}}\idx{2}})$\;

\If{$\vc{k}\trans\left( \vc{\hat{p}}\idx{1} \times \vc{\hat{p}}\idx{2} \right) < 0$}{
    $\theta\opt \gets -\theta\opt$ // Correct sign of rotation if needed \;
}
\KwReturn{$\theta\opt$}
\end{algorithm}

\begin{algorithm}[ht]
\DontPrintSemicolon 
\caption{Subproblem 2 \\
$\{(\theta\idx{1}\opt,\theta\idx{2}\opt)\idx{j}\}_{j=1}^2 \gets \subproblem{2}(\vc{p}\idx{1}, \vc{p}\idx{2},\vc{k}\idx{1},\vc{k}\idx{2})$ \\
Find a set of angles to rotate vector $\vc{p}\idx{1}$ about axis $\vc{k}\idx{1}$ and vector $\vc{p}\idx{2}$ about axis $\vc{k}\idx{2}$ to minimize the 2-norm error between the rotated vectors}
\label{alg: subproblem 2}
\setstretch{1.1}

\nl \KwInput{$\vc{p}\idx{1}, \vc{p}\idx{2}, \vc{k}\idx{1}, \vc{k}\idx{2}$}

\nl // Normalize input vectors \;

$\hat{\vc{p}}\idx{1} \gets \normalize(\vc{p}\idx{1})$
and
$\hat{\vc{p}}\idx{2} \gets \normalize(\vc{p}\idx{2}) $\;

\nl // Use Subproblem 4 for each angle; the first angle solution set is indexed by $i$ and the second by $j$ and each may have 1 or 2 solutions \;

$\{\theta\idx{1,1}\opt, \theta\idx{1,2}\opt\} \gets \subproblem{4}(\vc{k}\idx{2}, \hat{\vc{p}}\idx{1}, \vc{k}\idx{1}, \vc{k}\idx{2}\trans \hat{\vc{p}}\idx{2})$ \;

$\{\theta\idx{2,1}\opt, \theta\idx{2,2}\opt \} \gets \subproblem{4}(\vc{k}\idx{1}, \hat{\vc{p}}\idx{2}, \vc{k}\idx{2}, \vc{k}\idx{1}\trans \hat{\vc{p}}\idx{1})$ \;

\nl // Return set of 1, 2 solutions \;

\KwReturn{$\{ 
(\theta\idx{1,1}\opt,\theta\idx{2,2}\opt), 
(\theta\idx{1,2}\opt,\theta\idx{2, 1}\opt) \}$}
\end{algorithm}

\begin{algorithm}[ht]
\DontPrintSemicolon 
\caption{Subproblem 4 \\
$\{\theta\opt\idx{j}\} \gets \subproblem{4}(\vc{p},\vc{h},\vc{k},d)$ \\
Find a set of angles $\{\theta\opt\idx{j}\}$ by which to rotate vector $\vc{p}$ about axis $\vc{k}$ to minimize the distance of the rotated vector to a plane with normal $\vc{h}$ offset from the origin by distance $d$}
\label{alg: subproblem 4}
\setstretch{1.1}
 \KwInput{$\vc{p},\vc{h},\vc{k},d$}

 // Define plane parallel to the circle traced by rotating $\vc{p}$ about $\vc{k}$, into which we project to find a solution; columns of $\vc{F}$ are a basis of this plane \;

$\vc{F} \gets \mat{\skewmat(\vc{k})\vc{p} \ ,\  -\skewmat(\vc{k})^2\vc{p}}$ \;

 // Set up linear system $\vc{A}\vc{x} = b$ where $\vc{A} \in \R^{1\times 2}$ and $\vc{x} = [\sin\theta,\cos\theta]\trans$ to allow solving for $\theta$\;

$\vc{A} \gets \vc{h}\trans \vc{F}$ and
$b \gets d - \vc{h}\trans\vc{k}\vc{k}\trans\vc{p}$ \;

$\vc{x} \gets \vc{A}\pinv b$ // Least-squares solution \;

// If least squares solution is not on the circle, two unique angles can rotate $\vc{p}$ to contact the plane \;

\eIf{$\norm{\vc{A}}_2^2 > b^2$}
{
    // Offset the least-squares solution $\vc{x}$ in the nullspace of $\vc{A}$ to get solutions on the circle \;
    
    $z \gets (\norm{\vc{A}}_2^2 - b^2)^{1/2}$ \;
    
    $\vc{x}\lbl{+}$ and $\vc{x}\lbl{--} \gets \vc{x} \pm z \mat{\vc{A}\arridx{2} \ ,\ 
    -\vc{A}\arridx{1}}$ \;
    
    $\theta\opt\lbl{+} \gets \atantwo(\vc{x}\lbl{+}\arridx{1},\vc{x}\lbl{+}\arridx{2})$ \;
    
    $\theta\opt\lbl{--} \gets \atantwo(\vc{x}\lbl{--}\arridx{1},\vc{x}\lbl{--}\arridx{2})$ \;
    
    \KwReturn{$\{\theta\opt\lbl{+}, \theta\opt\lbl{--}\}$}
}{
    // Least-squares solution is the only solution \;
    
    $\theta\opt \gets \atantwo(\vc{x}\arridx{1},\vc{x}\arridx{2})$ \;
    
    \KwReturn{$\{\theta\opt\}$}
}
\end{algorithm}
\subsection{Syncing Human and Robot Frames}\label{app: syncing human and robot body centric frames}

We propose a calibration-free procedure to map human and robot upper body and wrist keypoints into the robot's $0$th frame, which we treat as a shared reference frame; we call this the \textit{body-centric frame}.
This is necessary because human keypoints are typically given in an arbitrary frame defined by an input streaming device (e.g., VR headset or external camera).

\subsubsection{Creating the Body-Centric Frame}
To put human upper body keypoints in the $0$th frame, we create the transformations $(\orthoframe{\streaminput},\frameorigin{\streaminput})$, where ``$\streaminput$'' refers to the input streaming frame.
Suppose we are given left and right shoulder keypoints $\shoulder\lside\frm{\streaminput}, \shoulder\rside\frm{\streaminput}$ as input keypoints.
We define a lower body anchor point $\transvec\lbl{torso}\frm{\streaminput}$, typically a torso-fixed keypoint such as a vertebra position or the hip center, depending on the input streaming device.
Then, we create the body-centric frame with \Cref{alg: make frame}:
\begin{align}
    (\orthoframe{\streaminput},\frameorigin{\streaminput}) \gets \makeframe(\shoulder\lside\frm{\streaminput}, \shoulder\rside\frm{\streaminput}, \transvec\lbl{torso}\frm{\streaminput}).
\end{align}
Relative to the human torso, this treats the point between the shoulders as its origin, and front/left/up as the orthonormal frame direction order (i.e., ``X/Y/Z''), as shown in \Cref{fig: robot and human arm diagram}.

Humanoid robots typically come with a predefined upper body frame analogous to the human upper body frame, which we use as the $0$th or baselink frame.
However, if a robot uses a different convention, we must create the transformations $(\orthoframe{\robot}, \frameorigin{\robot})$, where ``$\robot$'' refers to the robot's default frame.
In this case, just as for the human, we define shoulder and torso keypoints  $(\shoulder\lside\frm{\robot}, \shoulder\rside\frm{\robot}, \transvec\frm{\robot})$ for the robot, then apply \Cref{alg: make frame} to redefine the $0$th frame:
\begin{align}
    (\orthoframe{\robot}, \frameorigin{\robot}) \gets \makeframe(\shoulder\lside\frm{\robot}, \shoulder\rside\frm{\robot}, \transvec\frm{\robot}).
\end{align}
Finally, we map all keypoints for the human and robot arms into the $0$th frame per \Cref{eq: frame transformation}.

\subsubsection{Creating Wrist Frames}
Different input streaming methods often use different conventions for wrist or hand orientation, necessitating the following preprocessing steps to get the hand orientation $\handorientation\frm{\streaminput}$.
Typically, we receive a hand orientation $\tilde{\handorientation}\frm{\streaminput}$ in the input frame with an arbitrary convention, but our method assumes that the hand orientation follows the right-hand rule (index finger extended / palm normal / thumb extended for ``X/Y/Z'').
To correct this, we apply an appropriate rotation $\rotmat\lbl{align}\frm{\streaminput}$ to get $\handorientation\frm{\streaminput}$ that obeys our assumption as
$\handorientation\frm{\streaminput} \gets \rotmat\lbl{align}\frm{\streaminput}\tilde{\handorientation}\frm{\streaminput}$
    
Often the hand orientation is not directly available as an input, but finger and wrist keypoints are.
In this case, we consider the index and little finger root keypoints $\indexfinger\frm{\streaminput}$ and $\pinkyfinger\frm{\streaminput}$ (typically at the intersection between the metacarpal bone and the proximal phalanges bone), and wrist $\wrist\frm{\streaminput}$.
Then we have
\begin{align}
    (\handorientation\frm{\streaminput},\transvec\frm{\streaminput}) \gets  \makeframe(\indexfinger\frm{\streaminput},\pinkyfinger\frm{\streaminput},\wrist\frm{\streaminput}).
\end{align}

Similar to the human, the robot's end effector mount orientation may not follow our right-hand rule convention, which we denote $\tilde{\toolorientation}(\config)$.
In this case, we again apply an appropriate rotation to get $\toolorientation(\config) \gets \rotmat\lbl{align}\frm{\robot}\tilde{\toolorientation}(\config)$.

\begin{algorithm}[ht]
\DontPrintSemicolon 
\caption{Make Frame from Keypoints \\
$(\rotmat,\transvec) \gets \makeframe(\vc{k}\lside,\vc{k}\rside,\vc{k}\lbl{b})$ \\
// Given three non-collinear keypoints in $\R^3$, output a coordinate frame parameterized by a rotation matrix and translation vector}
\label{alg: make frame}
\setstretch{1.1}

\KwInput{Left, right, and bottom keypoints $\vc{k}\lside,\vc{k}\rside,\vc{k}\lbl{b}$}

// Set translation vector between left/right keypoints

$\vc{p} \gets \tfrac{1}{2}(\vc{k}\lside + \vc{k}\rside)$

// Represent frame orientation as a rotation matrix

$\vc{u}\lbl{y} \gets \normalize(\vc{k}\lside - \vc{k}\rside)$

$\vc{u}\lbl{x} \gets \normalize(\vc{u}\lbl{y}\times (\vc{p} - \vc{k}\lbl{b}))$

$\vc{u}\lbl{z} \gets \vc{u}\lbl{x} \times \vc{u}\lbl{y}$

$\rotmat \gets [\vc{u}\lbl{x}, \vc{u}\lbl{y}, \vc{u}\lbl{y}]$

\KwReturn{Rotation matrix $\rotmat$ and translation vector $\transvec$}
\end{algorithm}

\subsection{Perpendicular Wrist}\label{app: perpendicular wrist}

For some robots, such as the Rainbow RB-Y1 and Unitree G1, the end effector mounting orientation is perpendicular to the final joint axis (not parallel as we assume in \Cref{subsubsec: robot arm preliminaries} or \Cref{fig: robot and human arm diagram}).
We adjust \Cref{alg: align wrist} as follows to accommodate this wrist type.
Instead of using Subproblems 1 and 2, perpendicular wrists can be solved with Euler angle decomposition.
We take the following steps to derive an Euler angle order ``A$_5/$A$_6$/A$_7$'' (e.g., ``X/Y/Z'' or ``X/Y/X'') of the final three joints.
First we represent the $7$\ts{th} frame orientation in the robot's $5$\ts{th} joint frame with the current joint configuration $\config$::
\begin{align}
    \rotmat\frms{5}{7}\lbl{des} &\gets \transbrac{\rotmat\frms{0}{5}(\config)} \rotmat\frms{0}{7}\lbl{des}
\end{align}
Next, we inspect the correspondence between the rotation axis $(\axis\idx{5}, \axis\idx{6}, \axis\idx{7})$ and the current $5$\ts{th} joint frame axis:
\begin{align}
    \vc{u}\idx{5} &= \axis\idx{5} \rotmat\frms{0}{5}(\config) \\
    \vc{u}\idx{6} &= (\rotmat\frms{5}{6}\local \axis\idx{6}) \rotmat\frms{0}{5}(\config) \\
    \vc{u}\idx{7} &= (\rotmat\frms{5}{6}\local \rotmat\frms{6}{7}\local \axis\idx{7}) \rotmat\frms{0}{5}(\config)
\end{align}
For each of the $5$th-$7$th joints, the order of the body frame Euler rotation depends on $\vc{u}\idx{i}$'s value $(i = 5,6,7)$:
\begin{align}
    \vc{u}\idx{i} &= \mat{1 & 0 & 0} \Rightarrow \regtext{A}\idx{i} \gets \regtext{``X''}  \\
    \vc{u}\idx{i} &= \mat{0 & 1 & 0} \Rightarrow \regtext{A}\idx{i} \gets \regtext{``Y''}  \\
    \vc{u}\idx{i} &= \mat{0 & 0 & 1} \Rightarrow \regtext{A}\idx{i} \gets \regtext{``Z''}, 
\end{align}
Then we use Euler decomposition to find the wrist angles
\begin{align}
    (\jointangle\idx{5}, \jointangle\idx{6}, \jointangle\idx{7}) \gets 
    \eulerangle(\rotmat\frms{5}{7}\lbl{des}, ``\regtext{A}\idx{5}/ \regtext{A}\idx{6}/\regtext{A}\idx{7}")
\end{align}
Finally, we compute $\toolorientation(\config)$ as in \Cref{subsubsec: robot arm preliminaries}.

\begin{rem}\label{rem: gimbal lock for perp wrist}
A limitation of this approach can arise when $\axis\idx{5}$ and $\axis\idx{7}$ become parallel, when $\jointangle\idx{6}=\pm 90\degree$ (i.e., gimbal lock).
However, we find that this rarely occurs during human to robot pose retargeting, because the active range of motion for a human wrist in the $\jointangle\idx{6}$ is typically well below $90\degree$ \cite{ryu1991functional}.
\end{rem}


\subsection{Optimality Proof}\label{app: optimality proof}

We take advantage of the fact that each closed-form subproblem algorithm is optimal (\Cref{lem: subproblem 4 optimality,lem: subproblem 2 optimality,lem: subproblem 1 optimality}).
Note, we work in reverse order from Subproblem 4 to 2 to 1, because the subproblems are ultimately used in this order in \Cref{alg: SEW-Mimic}.
The final result is then below in \Cref{prop: sew-mimic is optimal}.

\begin{lem}[Optimality of Solving Subproblem 4]
\label{lem: subproblem 4 optimality}
    Consider a vector $\vc{p}$, a unit vector $\vc{k}$ defining a rotation axis, and a unit vector $\vc{h}$ defining a plane offset from the origin by a distance $d$.
    \Cref{alg: subproblem 4} finds an optimal solution $\theta\opt$ to $\min_{\theta} \abs{\vc{h}\trans\rodrigues(\vc{k},\theta)\vc{p} - d}$.
\end{lem}
\begin{proof}
    See \cite[Appendix A, Subproblem 4]{elias2025ik} for a detailed proof by construction, which we summarize here.
    First, note that the rotation of $\vc{p}$ about $\vc{k}$ traces a circle in 3-D, which intersects the plane defined by $\vc{h}$ at either 0, 1, or 2 points.
    The algorithm constructs a basis $\vc{F}$ for the plane in which this circle is embedded.
    We then find the closest points on that circle to the line created by the intersection with the plane defined by $\vc{h}$ and $d$, as a least-squares solution.
    To see this, first note that $       \rodrigues(\vc{k},\theta) = \vc{k}\vc{k}\trans + \sin\theta \skewmat(\vc{k}) - \cos\theta\skewmat(\vc{k})^2$.
    Then $\vc{h}\trans\rodrigues(\vc{k},\theta)\vc{p} = d$ can be written as a linear equation:
    \begin{align}
        \underbrace{\mat{\vc{h}\trans\skewmat(\vc{k})\vc{p},-\vc{h}\trans\skewmat(\vc{k})^2\vc{p}}}_{\vc{A}}
        \underbrace{\mat{\sin\theta \\ \cos\theta}}_{\vc{x}}
        &= \underbrace{d -\vc{h}\trans\vc{k}\vc{k}\trans\vc{p}}_{b},
    \end{align}    
    In the case of 0 or 1 intersection points, the solution is given by least squares: $\vc{x} = \vc{A}\pinv b$ and $\theta\opt = \atantwo(\vc{x}\arridx{1},\vc{x}\arridx{2})$.
    In the case of 2 intersection points, the plane intersects the circle traced by $\vc{p}$ about $\vc{k}$, with radius $(\norm{\vc{A}}_2^2 + b^2)^{1/2}$.
    The least-squares solution $\vc{x}$ lies inside that circle.
    Thus, one can find solutions on the circle by moving orthogonal to $\vc{x}$ (i.e., in the nullspace of $\vc{A}$) by a distance $z = (\norm{\vc{A}}_2^2 - b^2)^{1/2}$, resulting in two optimizers $\theta\opt\lbl{+}$ and $\theta\opt\lbl{--}$.
\end{proof}

\begin{lem}[Optimality of Solving Subproblem 2]
\label{lem: subproblem 2 optimality}
    Consider two unit vectors $\vc{u}, \vc{v} \in \R^3$ and two rotation axes $\vc{k}\idx{1},\vc{k}\idx{2} \in \R^3$.
    \Cref{alg: subproblem 2} finds optimal values of $\theta\idx{1}$ and $\theta\idx{2}$ that minimize
    $\similarity(\rodrigues(\vc{k}\idx{1},\theta\idx{1})\vc{u},\rodrigues(\vc{k}\idx{2},\theta\idx{2})\vc{v})$, where $\similarity$ is as in \Cref{eq: similarity metric}.
\end{lem}
\begin{proof}
    From \cite[Appendix A, Subproblem 2]{elias2025ik}, we have that \Cref{alg: subproblem 2} returns $\theta\lbl{1}$ and $\theta\lbl{2}$ that minimize the quantity $\norm{\rodrigues(\vc{k}\idx{1},\theta\idx{1})\vc{u} - \rodrigues(\vc{k}\idx{2},\theta\idx{2})\vc{v}}_2$ exactly to 0; the key idea is to use two instances of Subproblem 4, then rely on \Cref{lem: subproblem 4 optimality}.
    Then, to show that an optimizer of the 2-norm is also an optimizer of $\similarity$, notice that, for unit vectors $\vc{a}$ and $\vc{b}$, if $\norm{\vc{a} - \vc{b}}_2 = 0$, then $1 - \vc{a}\cdot\vc{b} = 0$. 
\end{proof}

\begin{lem}[Optimality of Solving Subproblem 1]\label{lem: subproblem 1 optimality}
Consider two unit vectors $\vc{u}$ and $\vc{v} \in \R^3$, and a rotation axis $\vc{k} \in \R^3$.
Assume all three vectors are not collinear.
Then, \Cref{alg: subproblem 1} finds an optimal angle $\theta\opt$ that minimizes $\norm{\rodrigues(\vc{k},\theta\opt)\vc{u} - \vc{v}}$.
\end{lem}
\begin{proof}
See \cite[Appendix A, Subproblem 1]{elias2025ik}.
The key idea is to project $\vc{p}\idx{1}$ and $\vc{p}\idx{2}$ into the plane perpendicular to $\vc{k}$, then apply basic planar trigonometry.
\end{proof}

Finally, we prove \Cref{prop: sew-mimic is optimal} (restated here).
\myPropLabel*
\begin{proof}
    We consider each cost term in \Cref{prob: retargeting problem}, which are all nonnegative.
    From \Cref{alg: SEW-Mimic} (\Cref{line: align upper arm}) and \Cref{lem: subproblem 2 optimality} we have that
    $\similarity(\upperarm,\rotmat\frms{0}{3}(\config)\axis\idx{3}) = 0$.
    Similarly, from \Cref{alg: SEW-Mimic} (\Cref{line: align lower arm}) and \Cref{lem: subproblem 2 optimality} we have that
    $\similarity(\lowerarm,\rotmat\frms{0}{5}(\config)\axis\idx{5}) = 0$.
    It remains to show that $\norm{(\toolorientation(\config)\trans\handorientation)^{1/2} - \eye}\lbl{F} = 0$, which can be shown by demonstrating that $\toolorientation(\config) = \handorientation$.
    Per \Cref{alg: align wrist}, the claim follows from \Cref{lem: subproblem 2 optimality,lem: subproblem 1 optimality};
    the use of Subproblem 2 ensures that the first column of the tool and hand orientation matrices are equal ($\toolorientation(\config)\arridx{:,1} = \handorientation\arridx{:,1}$), then Subproblem 1 ensures that second and third columns are equal ($\toolorientation(\config)\arridx{:,2:3} = \handorientation\arridx{:,2:3}$).
\end{proof}

Note, in the case of a perpendicular wrist as in \Cref{app: perpendicular wrist} we still have optimality (assume no gimbal lock per \Cref{rem: gimbal lock for perp wrist}), because $\rotmat\frms{5}{7}\lbl{des}$ is a valid rotation matrix by construction and Euler angle decomposition is exact.
\subsection{Safety Filter Details}\label{app: safety filter details}

\subsubsection{Collision Checking}
For each relevant collision pair of capsules $\capsule\idx{i},\capsule\idx{j}$, let
\begin{align}\label{eq: capsule collision check}
    (d,\vc{n}\idx{i},\vc{n}\idx{j}) = \collisioncheck(\capsule\idx{i},\capsule\idx{j}),
\end{align}
where $d$ is the signed distance between the capsules (negative when the capsules are intersecting) and $\vc{n}\idx{i}$ is the contact normal on $\capsule\idx{i}$ and similarly $\vc{n}\idx{j}$.
We compute $d$ per \cite[\S5.9.1]{ericson2004real}.
We approximate the contact normals $\vc{n}\idx{i}$ as follows.
First, note that standard implementations of capsule collision checking return the quantities
$(\vc{c}\idx{i},\vc{c}\idx{j},\tau\idx{i},\tau\idx{j})$, where $\vc{c}\idx{i}$ is the closest point on the line segment defining $\capsule\idx{i}$ to the line segment defining $\capsule\idx{2}$, and $\vc{c}\idx{2}$ similarly and the values $\tau\idx{i}$ and $\tau\idx{2}$ define where these points occur on their corresponding line segments.
That is, if $\capsule\idx{i} = \capsule(\vc{p}\idx{i,1},\vc{p}\idx{i,2},\radius\idx{i})$, then
$\vc{c}\idx{i} = \tau\idx{i}\vc{p}\idx{i,1} + (1-\tau\idx{i})\vc{p}\idx{i,2}$.
Then we approximate the contact normal for $\capsule\idx{i}$
$\vc{n}\idx{i} = \tau\idx{i} \cdot \normalize(\vc{c}\idx{2} - \vc{c}\idx{1})$,
and $\vc{n}\idx{j}$ similarly.

\subsubsection{Continuous Time Approximation for Collision Check}
\label{app: cont time coll check}

As robots move in continuous time, collision checking must consider the full path between the current pose and the desired pose.
For our safety filter, we propose a simple continuous collision check by linearly interpolating between the current keypoints and the desired keypoints using the function $\linspace$, which takes start point, end point (inclusive), and number of interpolated points as the inputs.
We summarize this process in \Cref{alg: continuous time collision}.
We first determine the number of interpolated points $n\interp\in\N$ by the distance between the keypoints and the radius of the capsules (\Cref{line: get init and des keypoints}--\Cref{line: compute n interp}).
The key idea is to update the desired pose with the first interpolated keypoints where collision occurs, such that colliding bodies do not ``jump'' over each other.
We first interpolate between the robot's current and desired SEW keypoints, and use them to represent links (\Cref{line: safety filter start make keypoints and capsules}--\Cref{line: safety filter end make keypoints and capsules}).
We then identify the first instance of collision in the interpolation (\Cref{line: safety filter start first collision}--\Cref{line: safety filter end first collision}), and return the corresponding collision volumes.

\begin{algorithm}[ht]
\DontPrintSemicolon 
\caption{Continuous-Time Collision Check \\
$\collisionvolumes \gets \conttimecollision(\config\init,\config\des)$ \\
// Given initial and desired configurations}
\label{alg: continuous time collision}
\setstretch{1.1}

\KwInput{Initial pose $\config\init$ and desired pose $\config\des$}

// Determine number of interpolation points by getting initial and final keypoints, finding max distance between corresponding keypoints, then dividing by capsule radius \newline
$\keypoints\init \gets \FK(\config\init)$ and
$\keypoints\des \gets \FK(\config\des)$ \label{line: get init and des keypoints}

// Let $\vc{\radius}$ be a vector of capsule radii corresponding to the keypoints, and of appropriate size, then compute: \newline
$\vc{n} = \norm{\keypoints\init - \keypoints\des} / \vc{\radius}$ // operations are elementwise \label{line: compute possible interp values}

$n\interp \gets \ceil{\max_i \vc{n}\arridx{i}}$ \label{line: compute n interp}

// Interpolate robot arm keypoints and make capsules \label{line: safety filter start make keypoints and capsules}

$\{\robotkeypoints[i]\}_{i=0}^{n\interp} 
\gets \linspace(\FK(\config\init),\FK(\config\des),n\interp+1)$

$\{\collisionvolumes_i\}_{i=1}^{n\interp}
\gets \{\makecapsules(\robotkeypoints[i])\}_{i=1}^{n\interp}$ \label{line: safety filter end make keypoints and capsules}

// Find the first collided interpolation
\label{line: safety filter start first collision}

\eIf{
    $\exists\ i \in \{1, \cdots, n\interp\} \st \collisionvolumes_i$ is in collision
}{$i^* = \underset{i\in \{1, \cdots, n\interp\}}{\min}\{\collisionvolumes_i \regtext{ is in collision}\}$}
{$i^* = n\interp$}
$\collisionvolumes \gets \collisionvolumes_{i^*}$\label{line: safety filter end first collision}

\KwReturn{Collision volumes $ \collisionvolumes$}
\end{algorithm}

\subsubsection{XPBD Iteration}
The XPBD iteration (\Cref{line: safety filter xpbd update} of \Cref{alg: safety filter}) is detailed in \Cref{alg: XPBD update}.
The overall idea is to push the robot's capsules out of collision.
The algorithm operates by accumulating a gradient for each link's pair of keypoints, computed using collision normals between each pair of possibly colliding links.
Then, each keypoint is adjusted by its corresponding gradient, weighted by the Lagrange multiplier for that pair of links.
Finally, since this procedure may change the link lengths, we adjust the links back to their original lengths (\Cref{alg: XPBD length projection}); in this case, we again apply XPBD but per link.

\begin{algorithm}[ht]
\DontPrintSemicolon 
\caption{$(\collisionvolumes,\{\lambda\idx{i,j}\}) \gets \XPBDupdate(\collisionvolumes,\{\lambda\idx{i,j}\})$}
\label{alg: XPBD update}
\setstretch{1.1}

\KwInput{Collision volumes $\collisionvolumes$, one multiplier $\lambda\idx{i,j}$} for each collision volume pair

// Require: Safety margin $d\lbl{min}$, activation distance $d\lbl{act}$, release distance $d\lbl{rel}$, per-joint weights $\weight\idx{k}$, per-link lengths $\length\idx{k}$, compliance $\alpha$

\For{each collision pair $(\capsule\idx{i},\capsule\idx{j}) \in \collisionvolumes\times\collisionvolumes$}{
    // Denote $\capsule\idx{i}$ as parameterized by $(\vc{p}\idx{i,1},\vc{p}\idx{i,2},\radius\idx{i})$ and $\capsule\idx{j}$ similarly by $(\vc{p}\idx{j,1},\vc{p}\idx{j,2},\radius\idx{j})$
    
    // Get signed distances and contact normals

    $(d,\vc{n}\idx{i},\vc{n}\idx{k}) \gets \collisioncheck(\capsule\idx{i},\capsule\idx{j})$

    // Check if in collision or hysteresis
    
    \If{$d \geq d\lbl{rel}$ OR ($d \geq d\lbl{act}$ AND $\lambda\idx{i,j} = 0$)}{
        $\lambda_{i, j} \leftarrow 0$ and \textbf{continue}
    }
    
    $c = d - d\lbl{min}$
    
    \If{$c < 0$}{
    \For{$k = i,j$}{
           
        $\vc{g}\idx{k,1}$ and $\vc{g}\idx{k,2} \gets \zeros_{3\times 1}$ // Init. gradients
    
        \If{$\capsule\idx{k}$ is an upper arm}{
            // Only accumulate elbow gradient 
    
            $\vc{g}\idx{k,2} \gets \vc{g}\idx{k,2} + \vc{n}\idx{k}$
        }
        \ElseIf{$\capsule\idx{k}$ is a lower arm or hand}{
            // Accumulate both keypoint gradients
    
            $\vc{g}\idx{k,1} \gets \vc{g}\idx{k,1} + (1-\norm{\vc{n}\idx{k}}_2)\vc{n}\idx{k}$
    
            $\vc{g}\idx{k,2} \gets \vc{g}\idx{k,2} + \vc{n}\idx{k}$
        }
    
        // Update multipliers, keypoints, capsules
    
        $\lambda\lbl{old} \gets \lambda\idx{i,j}$

        $\Delta\lambda \gets -(c + \alpha\lambda\lbl{old})/(\alpha + \sum_k \weight\idx{k}\norm{\vc{g}\idx{k}}_2^2)$
    
        $\lambda\idx{i,j} \gets \max\left(0,
                \lambda\lbl{old} + \Delta\lambda
            \right)$
    
        $\vc{p}\idx{k,1} \gets \vc{p}\idx{k,1} + \weight\idx{k}(\lambda\idx{i,j} - \lambda\lbl{old})\vc{g}\idx{k,1}$
        
        $\vc{p}\idx{k,2} \gets \vc{p}\idx{k,2} + \weight\idx{k}(\lambda\idx{i,j} - \lambda\lbl{old})\vc{g}\idx{k,2}$

        $\capsule\idx{k} \gets \capsulefunc(\vc{p}\idx{k,1},\vc{p}\idx{k,2},\radius\idx{k})$
    }
  }
}

// Fix links if keypoint adjustment changed lengths


$\lambda\idx{k} \gets 0\ \forall\ k = 1,\cdots,7$ // initialize per-link multipliers

$\collisionvolumes \gets \XPBDLengthUpdate(\collisionvolumes, \{\lambda\idx{k}\})$ // \Cref{alg: XPBD length projection}

\KwReturn{Updated capsules $\collisionvolumes$ and multipliers $\{\lambda\idx{i,j}\}$}
\end{algorithm}

\begin{algorithm}[ht]
\DontPrintSemicolon 
\caption{
$\collisionvolumes \gets \XPBDLengthUpdate(\collisionvolumes, \{\lambda\idx{k}\})$
}
\label{alg: XPBD length projection}
\setstretch{1.1}
\KwInput{Collision volumes $\collisionvolumes$, one multiplier $\lambda \idx{k}$ per link}
// Require: Per-link lengths $l\idx{k}$, per-joint weights $\weight\idx{k}$, compliance $\alpha$

\For{each collision volume $\capsule_{k} \in \collisionvolumes$ and length $l_k$}{
    // Denote $\capsule\idx{k}$ as parameterized by $(\vc{a},\vc{b},\radius)$

    // Compute current length and constraint violation
    
    $\vc{d} \gets \vc{b} - \vc{a}$; \quad $\ell \gets \|\vc{d}\|_2$
    
    \If{$\ell < \epsilon$}{\textbf{continue}}
    
    $c \gets \ell - l_k$
    
    // Compute gradients
    
    $\vc{g}\lbl{b} \gets \vc{d} / \ell$ and $\vc{g}\lbl{a} \gets -\vc{g}\lbl{b}$
    
    // Update multiplier, endpoints
    
    $\lambda\lbl{old} \gets \lambda\idx{k}$
    
    $\Delta\lambda \gets -(c + \alpha\lambda\lbl{old}) / (\alpha + \weight\lbl{a}\|\vc{g}\lbl{a}\|_2^2 + \weight\lbl{b}\|\vc{g}\lbl{b}\|_2^2)$
    
    $\lambda\idx{k} \gets \lambda\lbl{old} + \Delta\lambda$
    
    // Apply position corrections
    
    $\vc{a} \gets \vc{a} + \weight\lbl{a} \Delta\lambda \, \vc{g}\lbl{a}$;
    $\vc{b} \gets \vc{b} + \weight\lbl{b} \Delta\lambda \, \vc{g}\lbl{b}$

    $\capsule\idx{k} \gets \capsulefunc(\vc{a},\vc{b},\radius)$
}

\KwReturn{Updated capsules $\collisionvolumes$} 
\end{algorithm}

\begin{algorithm}[ht]
\DontPrintSemicolon 
\caption{Recover Tool Orientation \\
$\handorientation \gets \recovertoolori(\toolorientation,\tool)$ \\
// Given a tool frame orientation matrix and a desired alignment vector for x-axis in $\R^3$, compute a new desired tool frame orientation matrix}
\label{alg: recover tool}
\setstretch{1.1}

\KwInput{Current tool orientation and desired alignment vector $\toolorientation,\tool$}

// obtain the tool frame pointing direction as x-axis

$\vc{u}\lbl{x} \gets \toolorientation[:, 1]$

$\vc{u}\lbl{t} \gets \normalize(\tool)$

// find difference between x-axis and alignment vector

$\vc{k} \gets \vc{u}\lbl{x} \times \vc{u}\lbl{t}$

$\theta \gets \arccos{(\vc{u}\lbl{x}\trans \vc{u}\lbl{t})}$

$\handorientation \gets \rodrigues(\vc{k},\theta) \toolorientation$

\KwReturn{Desired tool frame orientation matrix $\handorientation$}
\end{algorithm}
\subsection{Hardware Setup Details}\label{app: hardware setup}

\subsubsection{Kinova Gen3 Dual Arm}
We use a custom-built dual Kinova Gen3 7-DOF arm setup, controlled by a laptop workstation with 24 core Intel i9 CPU and NVIDIA RTX 4090 GPU.
Each arm is equipped with a Psyonic Ability hand (6-DOF) as its end effector, each controlled by Teensy 4.0 microcontrollers and the Psyonic Ability API.
The robot has a StereoLabs ZED X binocular camera as its head, mounted atop consisting of a 2-DOF neck (two Dynamixel XC330-M288-T servos), interfaced with the rest of the robot through an NVIDIA Jetson Orin NX (also used for video streaming).
The depth camera streams to a Meta Quest 3 VR headset.
For a user to control the head, camera azimuth and elevation from the headset are sent as desired joint angles to the neck servos.

\subsubsection{Rainbow Robotics RB-Y1}

We used RB-Y1m v1.3 equipped with an omni-directional base and a pair of full spherical perpendicular wrist.
The robot has a 2 DoF neck with a StereoLabs ZED 2 binocular camera. 

In demo videos, we also used RB-Y1a v1.1 equipped with a differential drive base and a pair of spherical parallel wrist (similar to Kinova Gen3).
A pair of 12 DoF X-Hand was used in demo videos and the hand was controlled using Open-Television \cite{cheng_open-television_2024}.

\subsection{Retargeting Performance Additional Details}\label{app: retargeting exp details}

We run this experiment using MuJoCo simulator \cite{todorov2012mujoco} on a desktop workstation with Intel(R) Core(TM) i9-13900K CPU, NVIDIA RTX 4090 GPU, and 128 GB RAM.

On the same dataset discussed in the main text, we compare \ourmethod and \method{GMR} on the metric from \cite[Eq. (1)]{araujo2025retargeting}, as shown in \Cref{fig: gmr loss evaluation}.
As expected, \ourmethod is significantly worse in this metric.

\begin{figure}[t]
    \centering
    \includegraphics[width=0.5\linewidth]{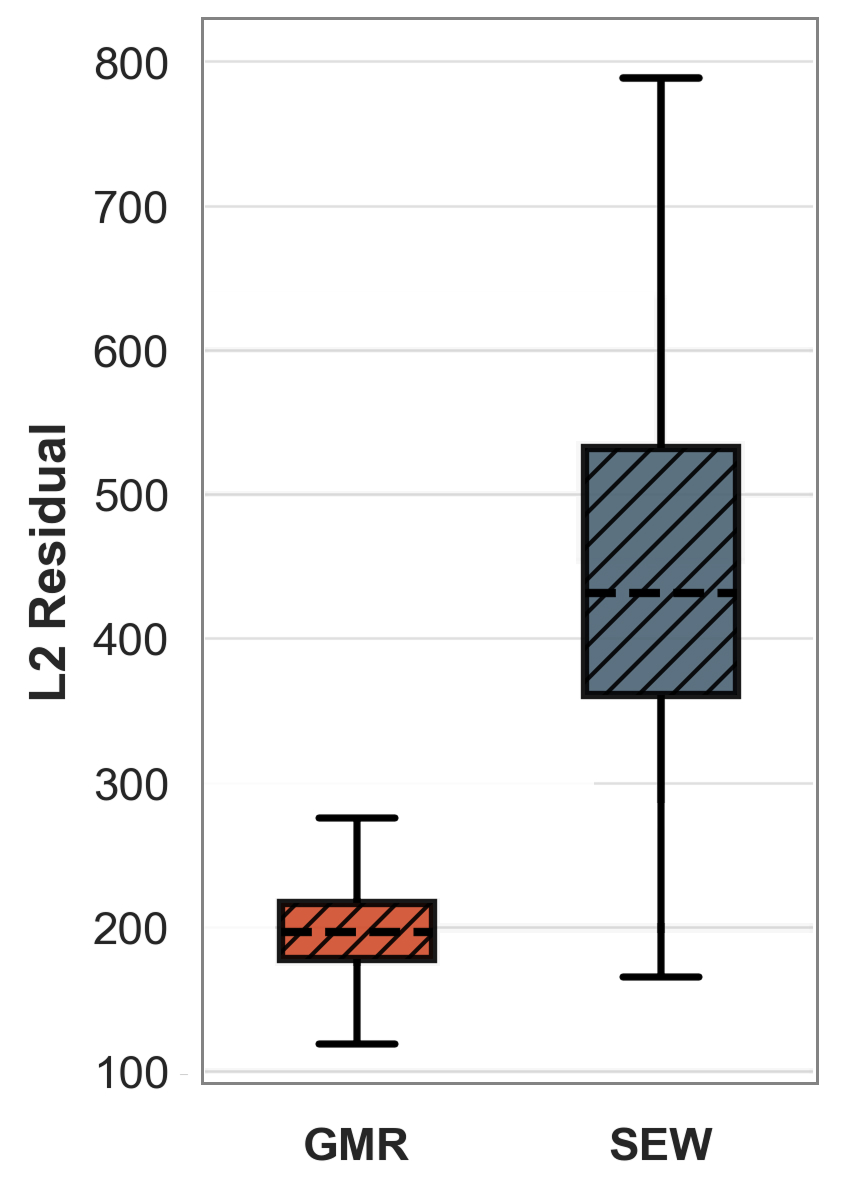}
    \caption{
    Evaluation of \ourmethod and \method{GMR} according to the \method{GMR}'s optimization loss formulation \cite[Eq. (1)]{araujo2025retargeting}. 
    } 
    \label{fig: gmr loss evaluation}
\end{figure}

\subsection{User Study Additional Details}\label{app: user study details}

We run this experiment using Robosuite \cite{zhu2020robosuite} on a desktop workstation with Intel(R) Core(TM) i9-13900K CPU, NVIDIA RTX 4090 GPU, and 128 GB RAM.

All tasks are visualized in \Cref{fig: task environment showcase}.

Before testing each control method, we allow the user five minutes to accustom themselves to the control method. 
Before each teleoperation task, we provide verbal instructions which informs users the goal of the task and obstacles to avoid. 
Each user is given five minutes to complete the task successfully in as many attempts as they are able to do.  
During user operation for each task attempt, we measure whether the goal was achieved successfully or not.
Task attempts were marked as successful when the goal object (the red cube) entered the goal area (the translucent green cube).
Attempts were alternatively marked as failures when the goal object reached an unrecoverable position, such as having fallen off the table, or when time for completion runs out during the trial.
In the case of Glass Gap (\Cref{fig: glass gap sew vs mink}), we add an additional case of marking the attempt as a failure when either stack of wine glasses are knocked over by the robot arm. 

\begin{table*}[ht]
\centering
\caption{User study data (bold values are best per task/controller).}
\label{tab: user study data}
\begin{tabular}{c|llll|llll|llll|}
\rowcolor[HTML]{C0C0C0} 
\cellcolor[HTML]{C0C0C0}                                     & \multicolumn{4}{c|}{\cellcolor[HTML]{C0C0C0}\textbf{Cabinet}}                                                                                                                                                                                         & \multicolumn{4}{c|}{\cellcolor[HTML]{C0C0C0}\textbf{Glass Gap}}                                                                                                                                                                                       & \multicolumn{4}{c|}{\cellcolor[HTML]{C0C0C0}\textbf{Handover}}                                                                                                                                                                                        \\ \cline{2-13} 
\rowcolor[HTML]{C0C0C0} 
\cellcolor[HTML]{C0C0C0}                                     & \multicolumn{2}{c|}{\cellcolor[HTML]{C0C0C0}\textbf{SEW}}                                                                 & \multicolumn{2}{c|}{\cellcolor[HTML]{C0C0C0}\textbf{MINK}}                                                                & \multicolumn{2}{c|}{\cellcolor[HTML]{C0C0C0}\textbf{SEW}}                                                                 & \multicolumn{2}{c|}{\cellcolor[HTML]{C0C0C0}\textbf{MINK}}                                                                & \multicolumn{2}{c|}{\cellcolor[HTML]{C0C0C0}\textbf{SEW}}                                                                 & \multicolumn{2}{c|}{\cellcolor[HTML]{C0C0C0}\textbf{MINK}}                                                                \\ \cline{2-13} 
\rowcolor[HTML]{C0C0C0} 
\multirow{-3}{*}{\cellcolor[HTML]{C0C0C0}\textbf{User}}      & \multicolumn{1}{c}{\cellcolor[HTML]{C0C0C0}\textbf{Success}} & \multicolumn{1}{c|}{\cellcolor[HTML]{C0C0C0}\textbf{Fail}} & \multicolumn{1}{c}{\cellcolor[HTML]{C0C0C0}\textbf{Success}} & \multicolumn{1}{c|}{\cellcolor[HTML]{C0C0C0}\textbf{Fail}} & \multicolumn{1}{c}{\cellcolor[HTML]{C0C0C0}\textbf{Success}} & \multicolumn{1}{c|}{\cellcolor[HTML]{C0C0C0}\textbf{Fail}} & \multicolumn{1}{c}{\cellcolor[HTML]{C0C0C0}\textbf{Success}} & \multicolumn{1}{c|}{\cellcolor[HTML]{C0C0C0}\textbf{Fail}} & \multicolumn{1}{c}{\cellcolor[HTML]{C0C0C0}\textbf{Success}} & \multicolumn{1}{c|}{\cellcolor[HTML]{C0C0C0}\textbf{Fail}} & \multicolumn{1}{c}{\cellcolor[HTML]{C0C0C0}\textbf{Success}} & \multicolumn{1}{c|}{\cellcolor[HTML]{C0C0C0}\textbf{Fail}} \\ \hline
\cellcolor[HTML]{EFEFEF}\textbf{1}                           & 2                                                            & \multicolumn{1}{l|}{2}                                     & 1                                                            & 2                                                          & 2                                                            & \multicolumn{1}{l|}{11}                                    & 0                                                            & 15                                                         & 3                                                            & \multicolumn{1}{l|}{3}                                     & 0                                                            & 7                                                          \\
\cellcolor[HTML]{EFEFEF}\textbf{2}                           & 6                                                            & \multicolumn{1}{l|}{1}                                     & 4                                                            & 0                                                          & 6                                                            & \multicolumn{1}{l|}{7}                                     & 3                                                            & 6                                                          & 6                                                            & \multicolumn{1}{l|}{4}                                     & 8                                                            & 0                                                          \\
\cellcolor[HTML]{EFEFEF}\textbf{3}                           & 10                                                           & \multicolumn{1}{l|}{0}                                     & 8                                                            & 0                                                          & 5                                                            & \multicolumn{1}{l|}{5}                                     & 0                                                            & 17                                                         & 4                                                            & \multicolumn{1}{l|}{2}                                     & 9                                                            & 0                                                          \\
\cellcolor[HTML]{EFEFEF}\textbf{4}                           & 3                                                            & \multicolumn{1}{l|}{1}                                     & 1                                                            & 2                                                          & 1                                                            & \multicolumn{1}{l|}{7}                                     & 0                                                            & 11                                                         & 6                                                            & \multicolumn{1}{l|}{0}                                     & 0                                                            & 2                                                          \\
\cellcolor[HTML]{EFEFEF}\textbf{5}                           & 8                                                            & \multicolumn{1}{l|}{0}                                     & 6                                                            & 0                                                          & 3                                                            & \multicolumn{1}{l|}{3}                                     & 0                                                            & 9                                                          & 5                                                            & \multicolumn{1}{l|}{0}                                     & 3                                                            & 0                                                          \\
\cellcolor[HTML]{EFEFEF}\textbf{6}                           & 7                                                            & \multicolumn{1}{l|}{0}                                     & 3                                                            & 0                                                          & 1                                                            & \multicolumn{1}{l|}{5}                                     & 1                                                            & 5                                                          & 3                                                            & \multicolumn{1}{l|}{1}                                     & 0                                                            & 4                                                          \\
\cellcolor[HTML]{EFEFEF}\textbf{7}                           & 8                                                            & \multicolumn{1}{l|}{0}                                     & 8                                                            & 0                                                          & 6                                                            & \multicolumn{1}{l|}{2}                                     & 2                                                            & 9                                                          & 7                                                            & \multicolumn{1}{l|}{1}                                     & 4                                                            & 0                                                          \\
\cellcolor[HTML]{EFEFEF}\textbf{8}                           & 10                                                           & \multicolumn{1}{l|}{0}                                     & 6                                                            & 0                                                          & 5                                                            & \multicolumn{1}{l|}{1}                                     & 0                                                            & 16                                                         & 8                                                            & \multicolumn{1}{l|}{2}                                     & 1                                                            & 4                                                          \\ \hline
\multicolumn{1}{r|}{\cellcolor[HTML]{EFEFEF}\textbf{total:}} & \textbf{54}                                                  & \multicolumn{1}{l|}{4}                                     & 37                                                           & 4                                                          & \textbf{29}                                                  & \multicolumn{1}{l|}{\textbf{41}}                           & 6                                                            & 88                                                         & \textbf{42}                                                  & \multicolumn{1}{l|}{\textbf{13}}                           & 25                                                           & 17                                                        
\end{tabular}
\end{table*}

\begin{figure}
    \centering
    \includegraphics[width=0.99\linewidth]{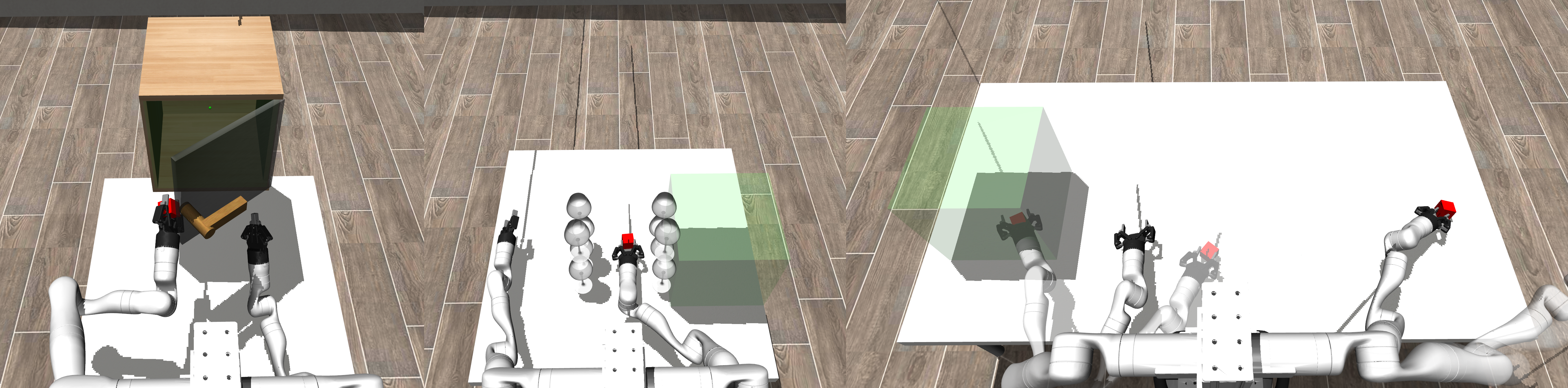}

    \includegraphics[width=0.99\linewidth]{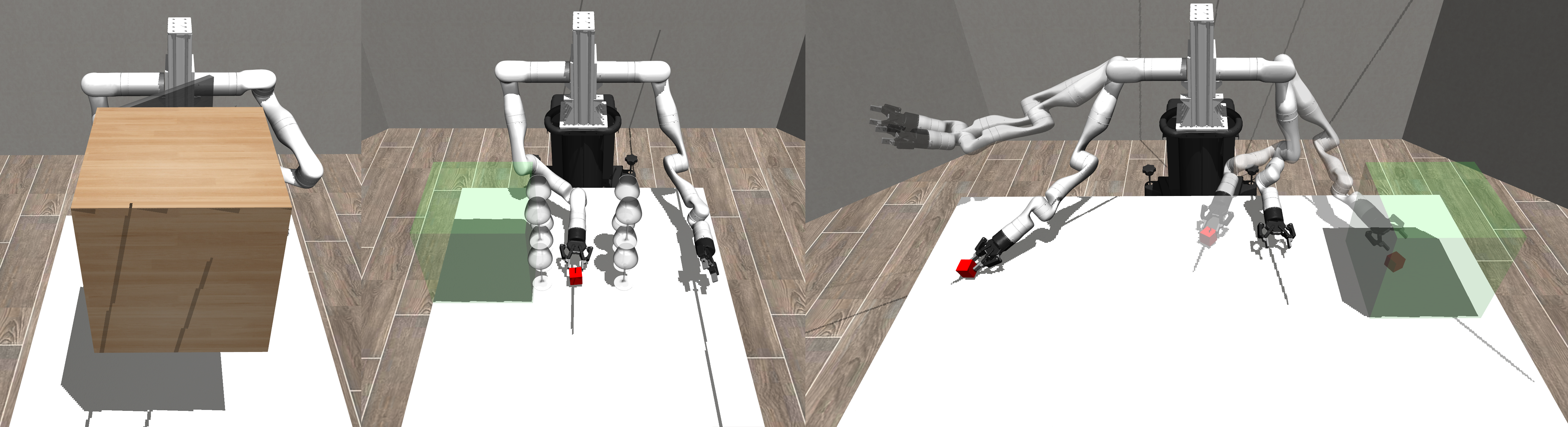}

    \includegraphics[width=0.99\linewidth]{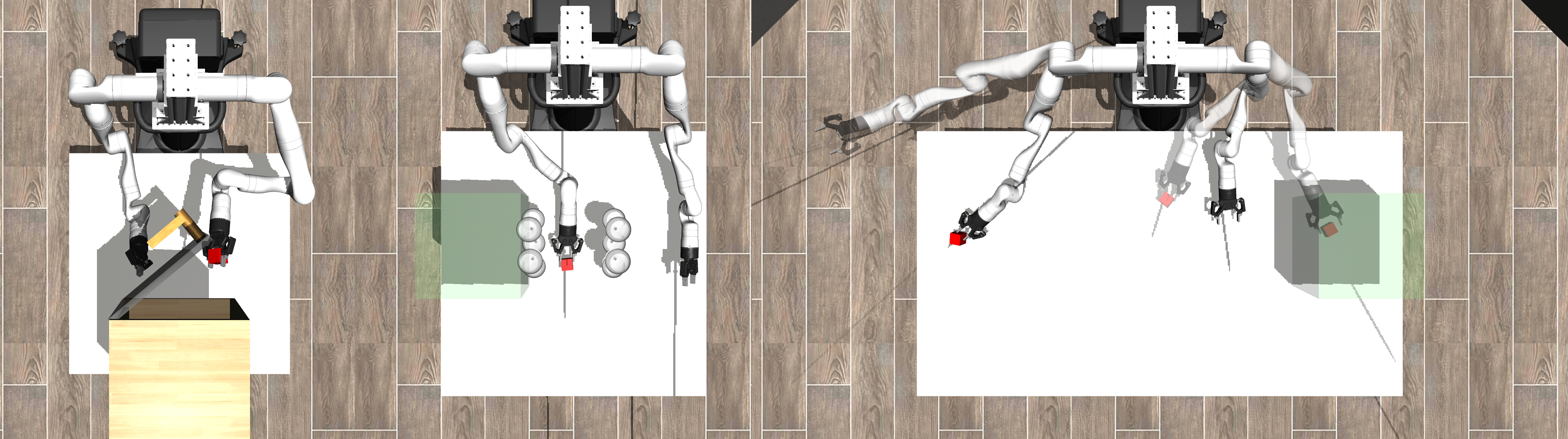}
    \caption{User experiment task environments. 
    We provide three tasks for users to complete during the user study, as described in \Cref{subsec: exp: user study}; each column shows one task from multiple views.
    Left: Cabinet.
    Middle: Glass Gap.
    Right: Handover.}
    \label{fig: task environment showcase}
\end{figure}

\subsection{Safety Filter Evaluation Additional Details}\label{app: safety filter evaluation details}

\begin{figure}[ht]
    \centering
    \begin{subfigure}{0.99\columnwidth}
    \includegraphics[width=0.99\columnwidth]{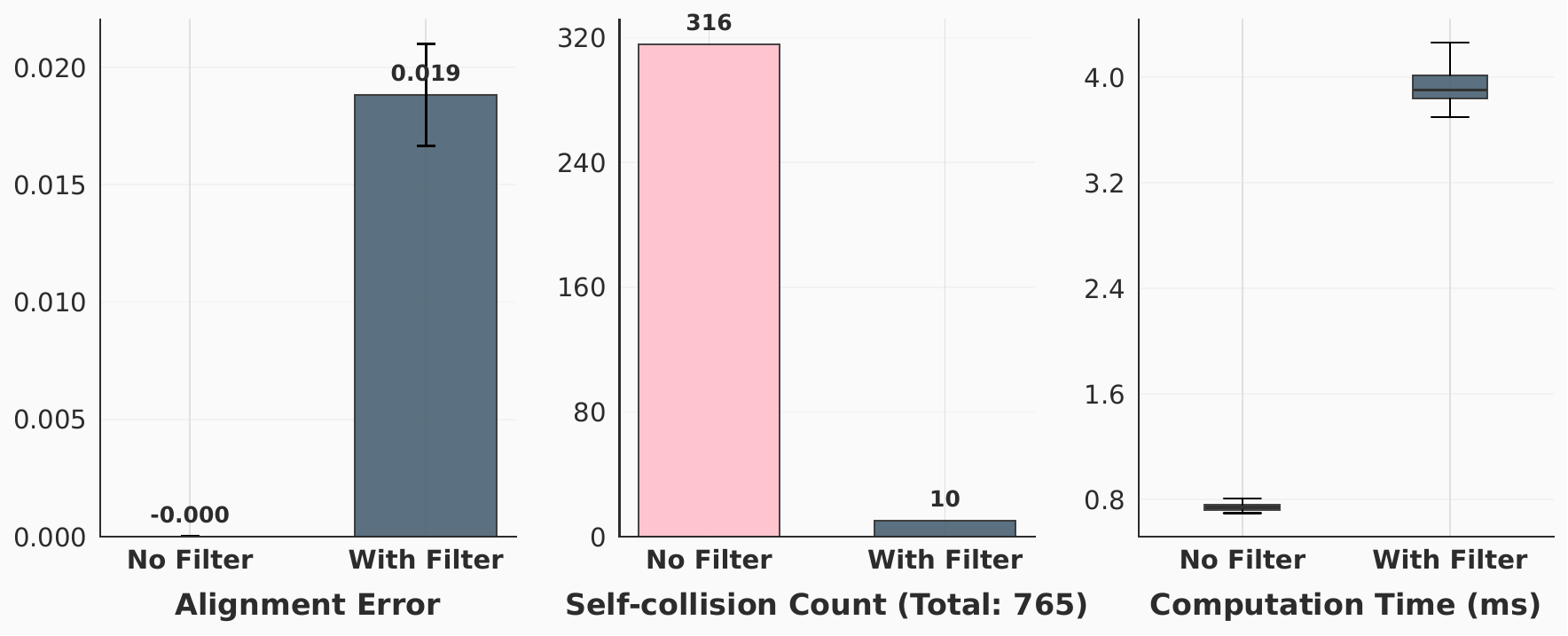}
    \end{subfigure}
    
    \caption{\ourmethod alignment error, collision instances, and computation time abolition study on whether safety filter is used in the pose-retargeting pipeline.}
    \label{fig: safety filter abolition}
\end{figure}

We run this experiment using MuJoCo simulator on a laptop workstation with Intel(R) Core(TM) i9-13950HK CPU, NVIDIA RTX 4090 GPU, and 64 GB of RAM.
Note that the GPU is not needed for \ourmethod or to run this experiment.

When creating capsules around the RB-Y1 robot upper body, we used the following capsule radius: torso - 0.15 m, shoulder - 0.12 m, upper arm - 0.06 m, lower arm - 0.065 m, wrist arm - 0.07 m. 
The length of each capsules are determined using their corresponding keypoints defined on the robot link. 

\subsection{Policy Learning Additional Details}\label{app: policy learning details}

\begin{figure}[ht]
    \centering
    \includegraphics[width=0.9\linewidth]{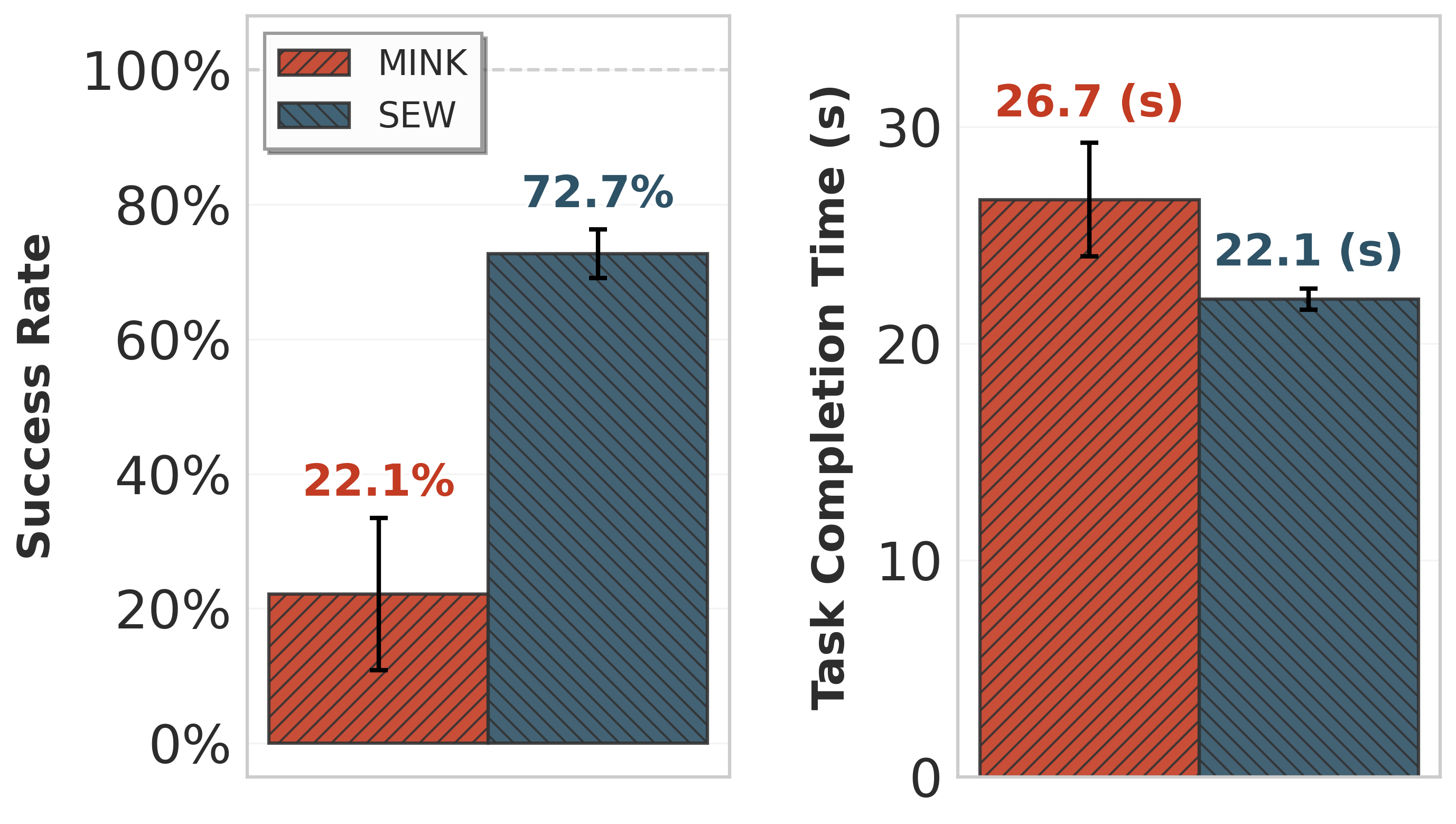}
    \caption{
    We train the same imitation learning policy on \textit{Glass Gap} demonstrations collected using \ourmethod or \basemink.
    \ourmethod yields higher success rate and lower task completion time.
    Bars show mean across training seeds; error bars indicate variability across seeds.
    } 
    \label{fig: glass_gap_eval}
    \vspace*{-0.5em}
\end{figure}

We run this experiment using MuJoCo simulator \cite{todorov2012mujoco} on a desktop workstation with Intel(R) Core(TM) i9-13900K CPU, NVIDIA RTX 4090 GPU, and 128 GB RAM.

For each representation, we trained for 600 epochs and selected the checkpoint with the best evaluation performance.
We report results over three random seeds, evaluated on 250 randomly sampled initial states.
To isolate the effect of action abstraction, we use the same number of demonstrations for both representations. 

\subsection{Full-Body Retargeting Additional Details}\label{app: exp: full body twist details}

We run this experiment using MuJoCo simulator on a desktop workstation with Intel(R) Core(TM) i9-13900K CPU, NVIDIA RTX 4090 GPU and 128 GB RAM.

We show in \Cref{fig: sew-twist-full-body-action} that \ourmethod provides dynamically feasible retargeting joint goals for TWIST without any fine-tuning.

\begin{figure}[ht]
    \centering
    \includegraphics[width=0.99\linewidth]{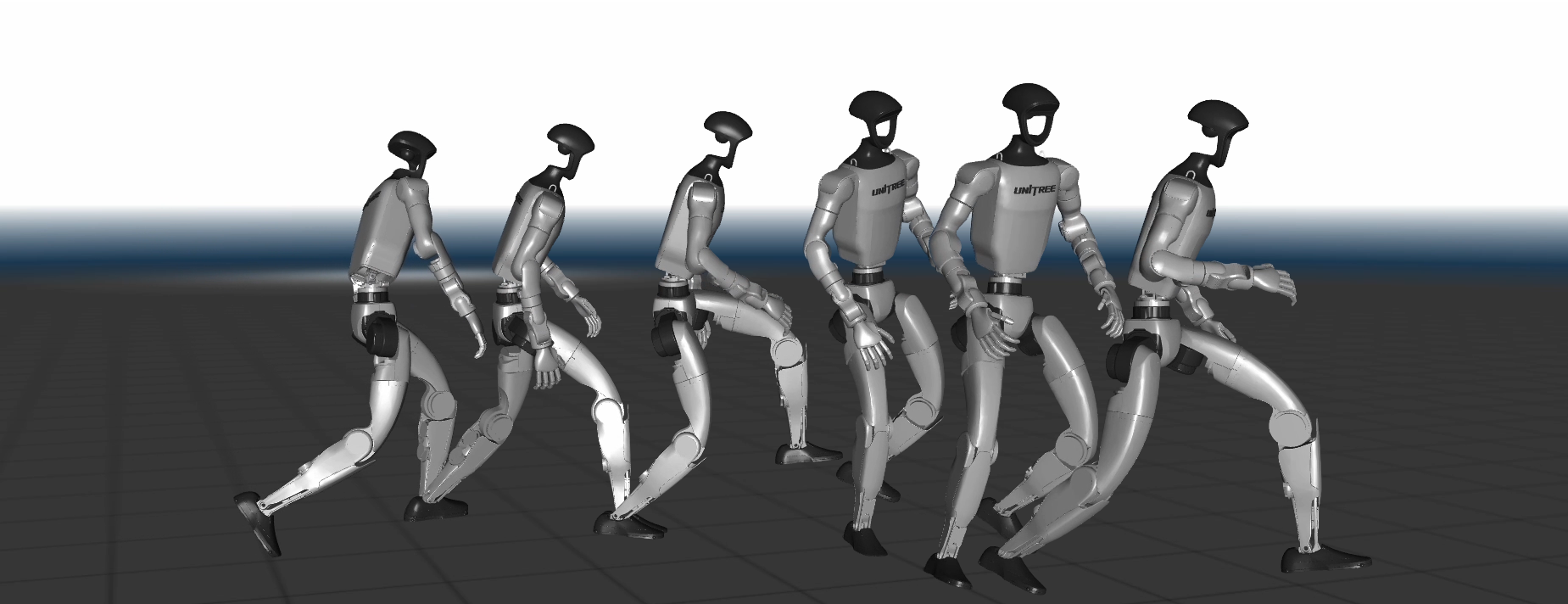}
    \caption{
    Inferencing LAFAN1 jumps1\_subject1 $(06:42 - 06:55)$ action on the SEW-TWIST pipeline for full body retargeting with dynamics.
    } 
    \label{fig: sew-twist-full-body-action}
    \vspace*{-0.5em}
\end{figure}




\subsection{Policy Learning Additional Experiment}

We run this experiment using MuJoCo simulator on a desktop workstation with Intel(R) Core(TM) i9-13900K CPU, NVIDIA RTX 4090 GPU.

\subsubsection{Hypothesis}
We hypothesize that using absolute SEW pose as an action abstraction improve sample efficiency and higher performance of trained policy.

\subsubsection{Experiment Design}
We process the \method{DexMimicGen} dataset \cite{jiang_dexmimicgen_2025} for the Fourier GR1 humanoid to extract SEW keypoints for the Pouring and Coffee tasks.
We then train a diffusion policy \cite{chi2025diffusion,reuss2023goal} $\policy_\theta(\action\idx{t}\!\mid\!\observation\idx{t})$ where the action $\action\idx{t}$ is the absolute SEW pose, and $\observation\idx{t}$ is the observation (RGB-D / proprioception).
We compare against the same policy trained on the same dataset with widely used action abstraction: desired end-effector poses with the default Mink representation \cite{jiang_dexmimicgen_2025}, and desired joint angles.
We measure task success and task completion time.

\subsubsection{Results and Discussion}
Task success rate is shown in \Cref{fig: policy learning success rate result}.
We see that the SEW representation performs similarly to data collected with \basemink on the Coffee task, but worse on the Pouring task; our preliminary investigation has not yet revealed the cause of this issue.
That said, the key takeaway we have is that the SEW representation warrants further investigation as an action representation.

\begin{figure}[ht]
    \centering
    \includegraphics[width=0.9\linewidth]{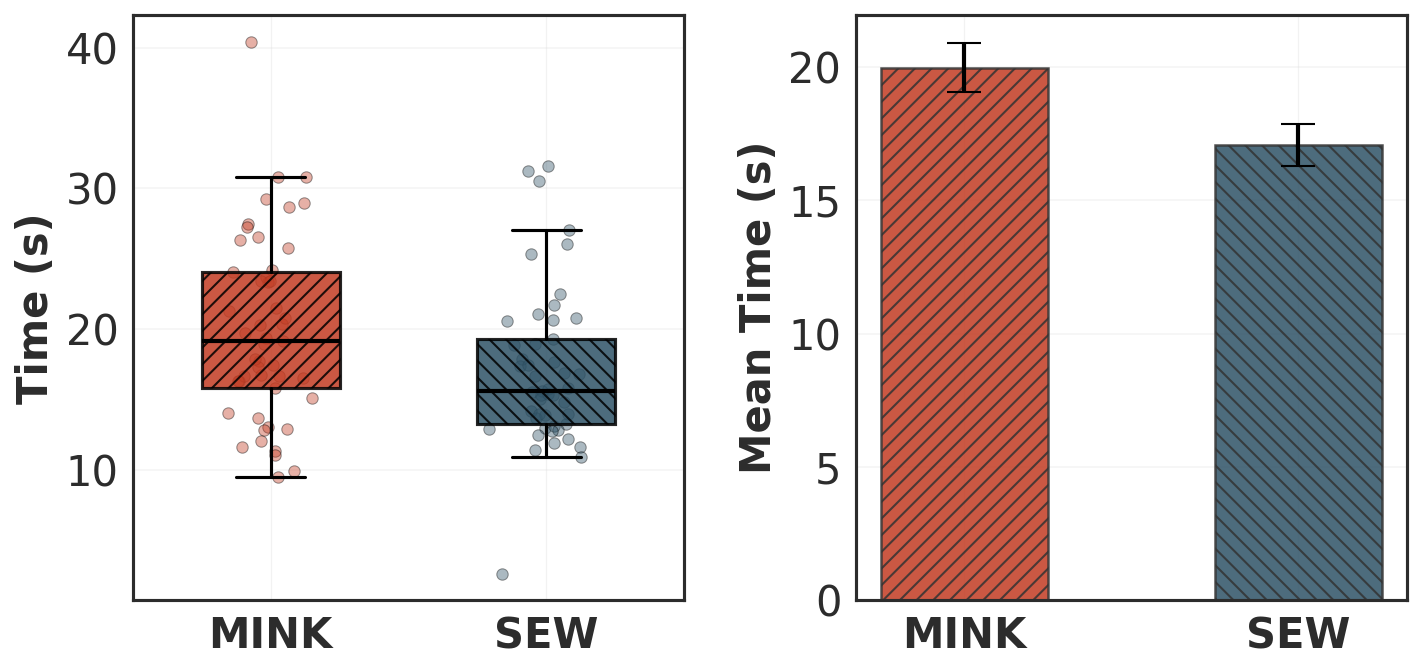}
    \caption{
    \textbf{Demo collection time.} Time required to collect a single demonstration with \basemink and \ourmethod. 
    The left panel shows the per-demonstration distribution (boxplot with individual trials); the right panel shows the mean with error bars indicating variability across demonstrations.
    }
    \label{fig:demo_collection_time}
\end{figure}

\begin{figure}[ht]
    \centering
    \includegraphics[width=0.9\linewidth]{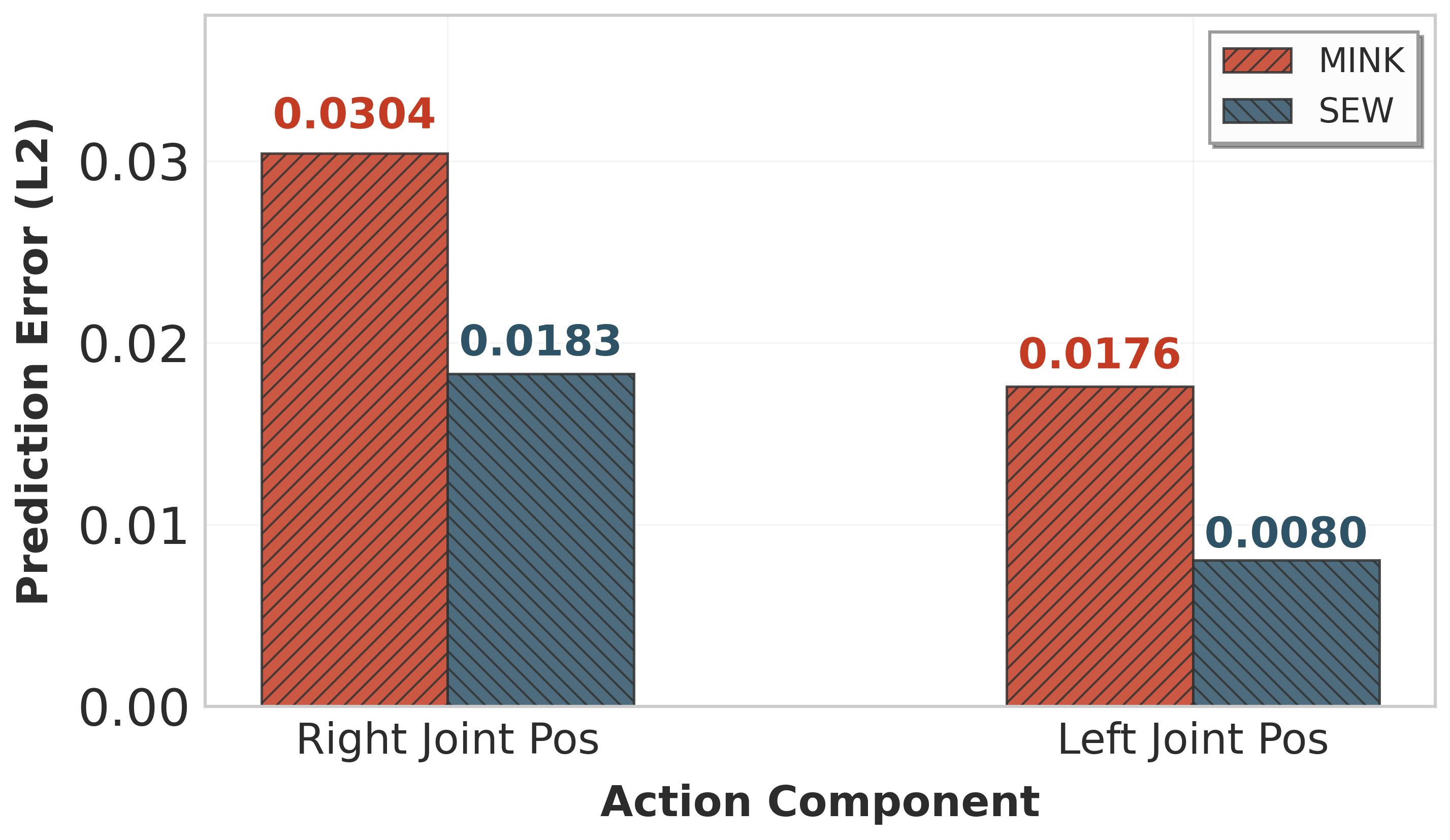}
    \caption{
    \textbf{Action prediction error.}
    We compare action prediction error for imitation policies trained on demonstrations collected with \basemink and \ourmethod. 
    From the expert dataset, we randomly sample observation--action-chunk pairs, predict the corresponding action chunk with the trained policy, and compute the mean squared error (MSE) between the predicted and ground-truth action chunks (reported as L2 error in the plot). 
    For each group, we evaluate the best-performing checkpoint. 
    We report errors separately for the left and right desired joint position components.
    }
    \label{fig:policy_learning_action_prediction_error}
\end{figure}

\begin{figure}[t]
    \centering

    \begin{subfigure}[t]{0.95\linewidth}
        \centering
        \includegraphics[width=\linewidth]{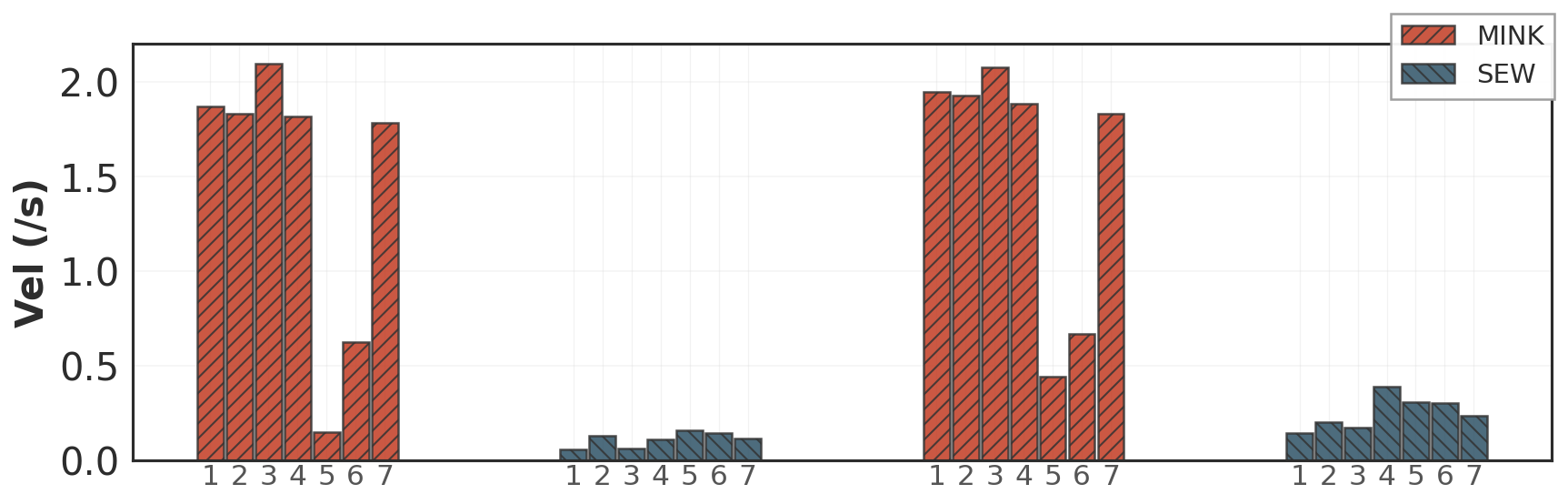}
        \caption{Per-joint commanded joint velocity magnitude for the left (L) and right (R) arms.}
        \label{fig:placeholder_top}
    \end{subfigure}

    \vspace{0.5em}

    \begin{subfigure}[t]{0.95\linewidth}
        \centering
        \includegraphics[width=\linewidth]{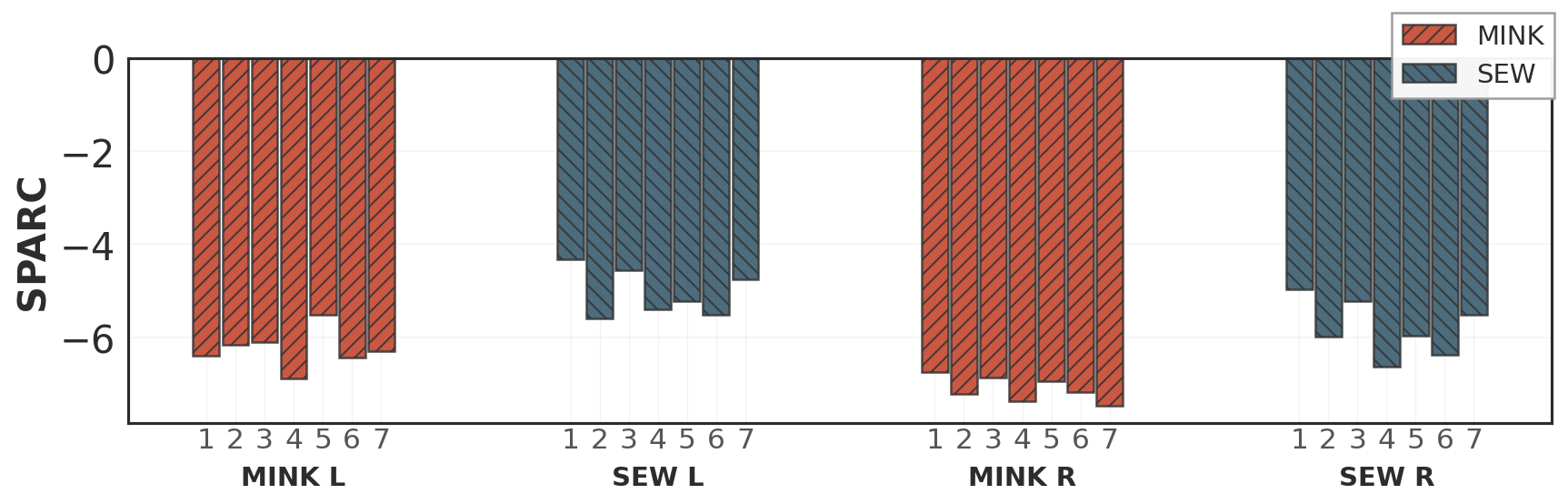}
        \caption{Per-joint trajectory smoothness measured by SPARC \cite{balasubramanian2015analysis} for the left (L) and right (R) arms.}
        \label{fig:placeholder_bottom}
    \end{subfigure}

    \caption{
    \textbf{Demonstration quality comparison.}
    We report (a) the magnitude of commanded joint velocity and (b) action-trajectory smoothness measured by SPARC \cite{balasubramanian2015analysis}, for both the left (L) and right (R) arms. Bars show the mean across demonstrations. Overall, \ourmethod produces lower commanded joint velocities and smoother trajectories (higher SPARC).}
    \label{fig:per_joint_comparison}
\end{figure}

\begin{figure}[ht]
    \centering
    \includegraphics[width=0.9\linewidth]{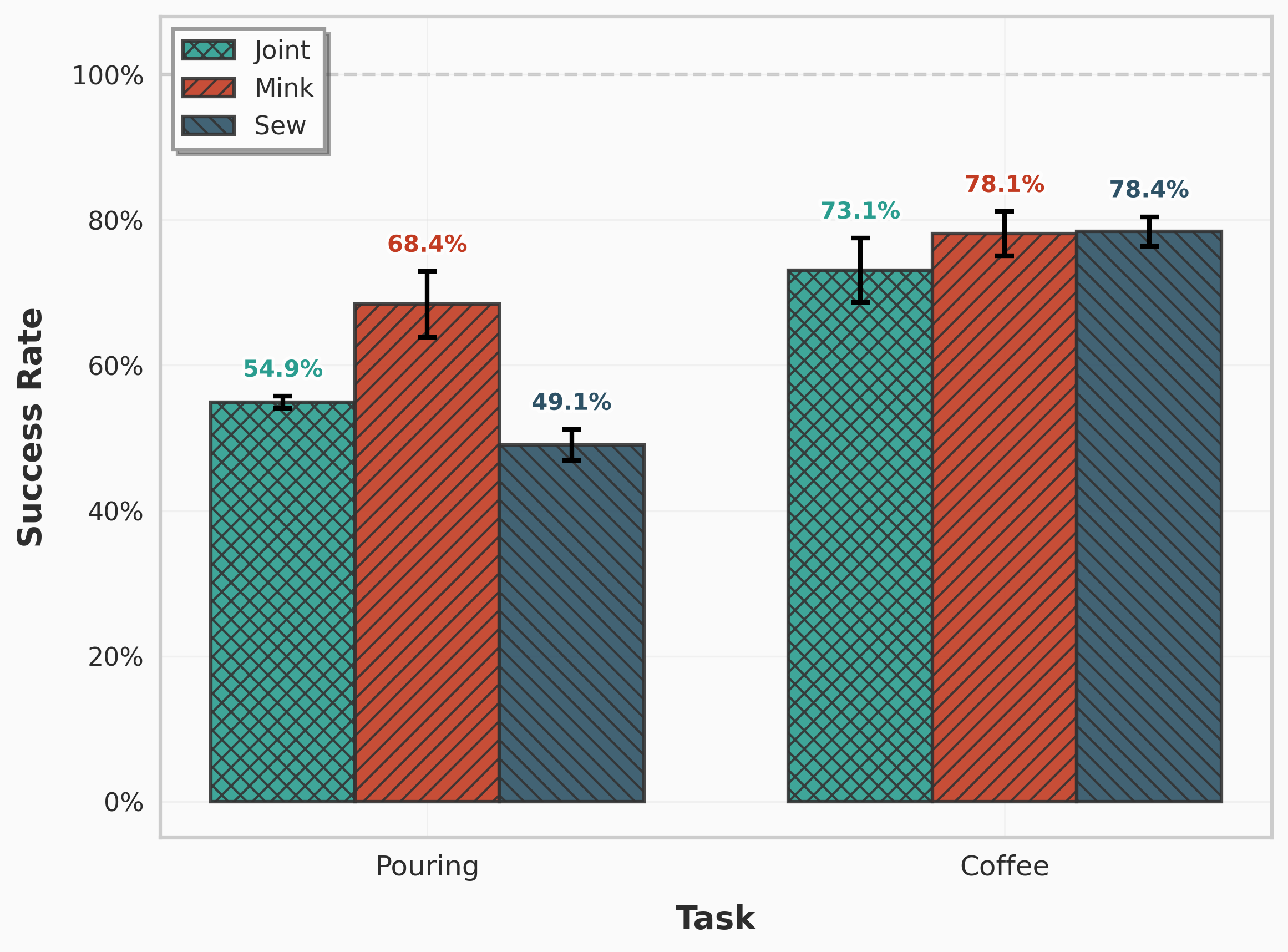}
    \caption{
    \textbf{Action abstraction ablation.} Success rate of Diffusion Policy \cite{chi2025diffusion} trained with three action representations—desired joint angles (Joint), end-effector pose commands via \basemink (Mink), and desired SEW (Sew)—on DexmimicGen tasks (Pouring, Coffee). Bars report mean performance across training seeds; error bars indicate variability across seeds.
    } 
    \label{fig: policy learning success rate result}
\end{figure}

\begin{figure}
    \centering
    \begin{subfigure}[t]{0.48\linewidth}
        \centering
        \includegraphics[width=\linewidth]{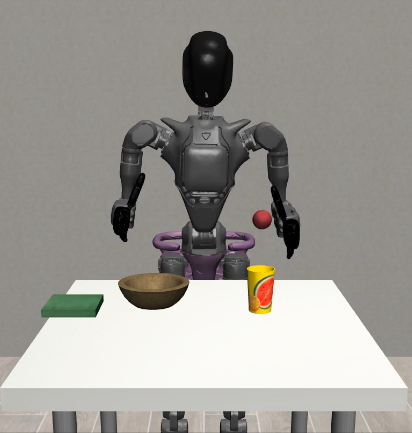}
        \caption{Pouring}
        \label{fig:pouring}
    \end{subfigure}\hfill
    \begin{subfigure}[t]{0.48\linewidth}
        \centering
        \includegraphics[width=\linewidth]{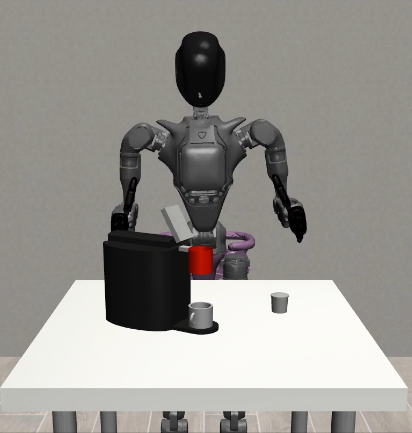}
        \caption{Coffee}
        \label{fig:coffee}
    \end{subfigure}

    \caption{
    \textbf{DexMimicGen tasks.}
    We evaluate on two bimanual manipulation tasks:
    (a) Pouring—transfer the ball from the yellow cup into the bowl, then place the bowl onto the green plate;
    (b) Coffee—grasp the coffee capsule, insert it into the coffee machine, and close the lid.
    }
    \label{fig:dexmimicgen_tasks}
\end{figure}

\begin{figure*}
    \centering
    \includegraphics[width=0.98\textwidth]{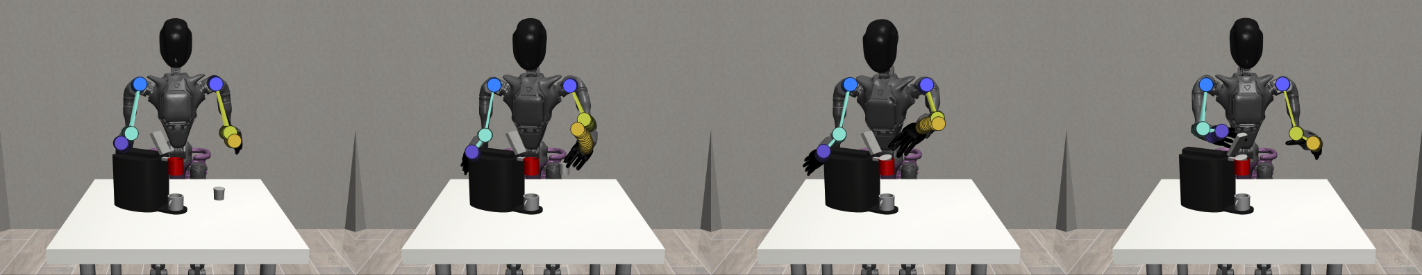}
    \caption{
    \textbf{Coffee task rollout with SEW actions.}
    We show four snapshots from a rollout of the Coffee task executed by a policy trained with the SEW action representation.
    Snapshots progress from the initial state (left) to task completion (right).
    Colored markers overlay the predicted SEW action poses at each timestep.
    }
    \label{fig:snapshot_sew_coffee}
\end{figure*}

\end{document}